\documentclass{article}

\usepackage[final]{neurips_2023}

\usepackage[utf8]{inputenc} %

\usepackage{xspace}
\usepackage{amsthm}
\usepackage{bbm}
\usepackage{algorithm}
\usepackage[algo2e, ruled]{algorithm2e}
\usepackage{mathrsfs}
\usepackage{bm}
\usepackage{makecell}
\usepackage{tabulary}

\usepackage[T1]{fontenc}    %
\usepackage{url}            %
\usepackage{booktabs}       %
\usepackage{amsfonts}       %
\usepackage{nicefrac}       %
\usepackage{microtype}      %
\usepackage{xcolor}         %
\usepackage{algpseudocode}
\usepackage{amsmath}

\DeclareMathOperator*{\argmax}{arg\,max}

\usepackage{natbib}
\usepackage{comment}
\newcommand{\poly}{\text{poly}}
\newcommand{\polylog}{\text{polylog}}
\newcommand{\norm}[1]{\|#1\|}

\renewcommand{\epsilon}{\varepsilon}
\newcommand{\R}{\mathbb{R}}
\newcommand{\Otilde}[1]{\widetilde{O}\left(#1\right)}
\newcommand{\otilde}[1]{\widetilde{O}(#1)}

\newcommand{\nc}{S}
\newcommand{\na}{K}
\newcommand{\ac}{j}
\newcommand{\lv}{n}

\newcommand{\rr}{\mathcal{R}_{\alpha}}
\newcommand{\dc}{\nu}
\newcommand{\db}{\omega}
\newcommand{\cs}{\mathcal{C}}
\newcommand{\cB}{\mathcal{B}}

\newcommand{\piout}{\pi^{\rm out}}
\newcommand{\rx}{A}
\newcommand{\rxh}{\widehat{A}}
\newcommand{\muu}{\overline{\mu}}
\newcommand{\mul}{\underline{\mu}}
\newcommand{\partition}{\mathcal{P}}

\newcommand{\nes}{L}

\newcommand{\ca}{g}
\newcommand{\cx}{i}
\newcommand{\cl}{b}
\newcommand{\cls}{\mathcal{B}}
\newcommand{\rw}{y}
\newcommand{\rk}{r}

\newcommand{\rg}{\text{Reg}}

\newcommand{\good}{\text{GOOD}}
\newcommand{\E}{\mathbb{E}}
\newcommand{\EE}[1]{\mathbb{E}[#1]}

\newcommand{\one}{\mathbbm{1}}
\newcommand{\wg}{w}
\newcommand{\rv}{v}
\newcommand{\lowrank}[1]{#1}
\newcommand{\cX}{\mathcal{X}}
\newcommand{\cA}{\mathcal{W}}
\newcommand{\cD}{\mathcal{D}}
\renewcommand{\lg}{\widetilde{\mathrm{lg}}}
\newcommand{\Lg}{16\log(rSK/\delta)}
\newcommand{\epst}{\widetilde{\epsilon}}

\newcommand{\unif}[1]{\mathrm{unif}(#1)}
\newcommand{\goodevent}{\mathcal{E}}

\usepackage{prettyref, hyperref}
\newcommand{\pref}[1]{\prettyref{#1}}

\newcommand{\savehyperref}[2]{\texorpdfstring{\hyperref[#1]{#2}}{#2}}
\newrefformat{eq}{\savehyperref{#1}{Eq. \textup{(\ref*{#1})}}}
\newrefformat{eqn}{\savehyperref{#1}{Equation~\ref*{#1}}}
\newrefformat{lem}{\savehyperref{#1}{Lemma~\ref*{#1}}}
\newrefformat{def}{\savehyperref{#1}{Definition~\ref*{#1}}}
\newrefformat{line}{\savehyperref{#1}{Line~\ref*{#1}}}
\newrefformat{thm}{\savehyperref{#1}{Theorem~\ref*{#1}}}
\newrefformat{corr}{\savehyperref{#1}{Corollary~\ref*{#1}}}
\newrefformat{cor}{\savehyperref{#1}{Corollary~\ref*{#1}}}
\newrefformat{sec}{\savehyperref{#1}{Section~\ref*{#1}}}
\newrefformat{app}{\savehyperref{#1}{Appendix~\ref*{#1}}}
\newrefformat{assum}{\savehyperref{#1}{Assumption~\ref*{#1}}}
\newrefformat{ex}{\savehyperref{#1}{Example~\ref*{#1}}}
\newrefformat{fig}{\savehyperref{#1}{Figure~\ref*{#1}}}
\newrefformat{alg}{\savehyperref{#1}{Algorithm~\ref*{#1}}}
\newrefformat{rem}{\savehyperref{#1}{Remark~\ref*{#1}}}
\newrefformat{conj}{\savehyperref{#1}{Conjecture~\ref*{#1}}}
\newrefformat{prop}{\savehyperref{#1}{Proposition~\ref*{#1}}}
\newrefformat{proto}{\savehyperref{#1}{Protocol~\ref*{#1}}}
\newrefformat{prob}{\savehyperref{#1}{Problem~\ref*{#1}}}
\newrefformat{que}{\savehyperref{#1}{Question~\ref*{#1}}}

\usepackage[capitalize]{cleveref}

\usepackage[textsize=tiny, disable]{todonotes}

\title{Context-lumpable stochastic bandits}

\usepackage{times}

\author{%
  Chung-Wei Lee\thanks{most works were done when interning at DeepMind.} \\
  University of Southern California\\
  \texttt{leechung@usc.edu} \\
  \And
  Qinghua Liu\\
  Princeton University\\
  \texttt{qinghual@princeton.edu}\\ 
  \And
  Yasin Abbasi-Yadkori\\
  Google DeepMind\\
  \texttt{yadkori@google.com}\\ 
  \And
  Chi Jin\\
Princeton University\\
\texttt{chij@princeton.edu}\\
  \And 
  Tor Lattimore\\
  Google DeepMind\\
  \texttt{lattimore@google.com}\\
  \And 
  Csaba Szepesv\'{a}ri\\
Google DeepMind and University of Alberta\\
\texttt{szepesva@ualberta.ca}\\
}

\newtheorem{theorem}{Theorem}%

\newtheorem{lemma}[theorem]{Lemma}
\newtheorem{proposition}[theorem]{Proposition}

\begin{document}

\maketitle

\begin{abstract}
We consider a contextual bandit problem with $\nc$ contexts and $\na$ actions.
In each round $t=1,2,\dots$ the learner
observes a random context and chooses an action based on its past experience.
The learner then observes a random reward whose mean is a function of the context and the action for the round.
Under the assumption that the contexts can be lumped into $r\le \min\{\nc,\na\}$ groups such that
the mean reward for the various actions is the same for any two contexts that are in the same group,
we give an algorithm that outputs an $\epsilon$-optimal policy after using at most $\widetilde O(\rk(\nc+\na)/\epsilon^2)$ samples with high probability and provide a matching $\Omega(\rk(\nc+\na)/\epsilon^2)$ lower bound.%
In the regret minimization setting, we give an algorithm whose cumulative regret up to time $T$ is bounded by $\widetilde O(\sqrt{\rk^3(\nc+\na)T})$.
To the best of our knowledge, we are the first to show the near-optimal sample complexity in the PAC setting and $\otilde{\sqrt{\poly(r)(S+K)T}}$ minimax regret in the online setting for this problem.  
\lowrank{We also show our algorithms can be applied to more general low-rank bandits and get improved regret bounds in some scenarios.}
\end{abstract}

\section{Introduction}

Consider a recommendation platform that interacts with a finite set of users in an online fashion.
Users arrive at the platform and receive a recommendation.
If they engage with the recommendation (e.g., they ``click'') then the platform receives a reward,
otherwise no reward is obtained.
Assume that the users can be partitioned into a small number of groups such that users in the same group
have the same preferences.
\begin{center}
\emph{
We ask whether
 there exist algorithms that can take advantage of the lumpability of users into a few groups, even when the identity of the group a user belongs to is unknown and only learnable because they share preferences with other users in the group.}
\end{center}

A slightly more general version of this problem can be formalized as follows:
Viewing users as ``contexts'' and recommendations as ``actions'' (or arms), assume that there are $\nc$ contexts and $\na$ actions.
In round $t=1,2,\dots$ the learner first receives a context $\cx_t$, sampled from an unknown distribution on the set $[\nc]:=\{1,\dots,\nc\}$ of possible contexts.
The learner then chooses an action $\ac_t\in [\na]:=\{1,\dots,\na\}$ and observes a reward
$\rw_t=\rx(\cx_t,\ac_t)+\eta_t$, where given the past, $\eta_t$ has a subgaussian tail (precise definitions are postponed to \cref{sec:defs}) and $\rx:[\nc]\times [\na]\to \R$ is an unknown function of mean rewards ($\R$ denotes the set of reals).%
We consider two settings when
the goal of the learner is either to identify a near-optimal policy $\pi: [\nc] \to [\na]$, or to keep its regret small.
Policy $\pi$ is called $\epsilon$-optimal if
\begin{align}
    \E[\rx(\cx_1,\pi(\cx_1))]\ge \max_{\pi'}\E[\rx(\cx_1,\pi'(\cx_1))]-\epsilon,
    \label{eq:sdef}
\end{align}
while the regret of the learner for a horizon of $T$ is
\begin{align}
    \rg_T=\E\left[\sum_{t=1}^T\max_{\ac\in[\na]}\rx(\cx_t,\ac)-\sum_{t=1}^T\rx(\cx_t,\ac_t)\right].
    \label{eq:rdef}
\end{align}
The expectations are taken with respect to the randomness of both the learner and environment, including contexts and rewards.
It is well known
(e.g., \cite{lattimore2020bandit}) %
that there are algorithms such that an $\epsilon$-optimal policy will be discovered after
\begin{align}
\widetilde O\left( \frac{\nc \na}{\epsilon^2} \right) \label{eq:bsac}
\end{align}
interactions, and there are also algorithms for which the regret satisfies
\begin{align}
\rg_T = \widetilde O( \sqrt{ \nc \na T } )\,. \label{eq:bret}
\end{align}
Here, the notation $\widetilde O(\cdot)$ hides polylogarithmic factors of the variables involved.
We say that the stochastic, finite, contextual bandit problem specified by $\rx$ is $\rk$-\emph{lumpable}
(or, in short, the bandit problem is \emph{context-lumpable})
if there is a partitioning of $[\nc]$ into
 $\rk\le \min\{\nc,\na\}$ groups such that $\rx(\cx,\cdot)=\rx(\cx',\cdot)$ holds whenever $\cx,\cx'\in [\nc]$ belong to the same group.
It is not hard to see that any algorithm needs at least $\Omega( \rk (\nc + \na)/\epsilon^2 )$ interactions to discover an $\epsilon$-optimal policy (\cref{thm:pac lower-bound}). %
Indeed, if we view $\rx$ as an $\nc \times \na$ matrix, the lumpability condition states that
$\rx = U V$ where $U$ is an $\nc \times \rk$ binary matrix where each row has a single nonzero element,
and $V$ is an $\rk \times \na$ matrix, which gives the unique mean rewards given the group indices. Hence, crude parameter counting suggests that there are $\rk(\nc+\na)$ parameters to be learned.
\begin{center}
\emph{
The question is whether $\nc \na$ in \cref{eq:bsac} and \cref{eq:bret} can be replaced with $(\nc+\na)\text{poly}(\rk)$ without knowing the grouping of the contexts.}
\end{center}

\lowrank{More generally, we can ask the same questions for contextual bandits with the \emph{low-rank} structure, where the matrix $\rx$ has rank $\rk$.
Equivalently, the low-rank condition means that we have the same decomposition $\rx= U V$ as above but no more constraints on $U$.
In the example of recommendation platforms, this assumption is more versatile as the users are modeled as \emph{mixtures} of $\rk$ preference types instead of belonging to one type only.
}

\subsection{Related works}
Our problem can be seen as an instance of contextual bandit problems introduced by \citet{auer2002nonstochastic}.
For a good summary of the history of the contextual bandit problem, the reader is advised to consult the article by \cite{TM17}.
Further review of prior works can also be found in the books of
\citet{slivkins2019introduction,lattimore2020bandit}.
\newcommand{\NE}{N}
Another way to approach context-lumpable stochastic bandits is to model them as stochastic linear bandits with changing action sets
\cite{Aue02,chu2011contextual,abbasi2011improved}.
However, lumpablility does not give improvements on the regret bound when running the standard algorithms in these settings (EXP4 and SupLinRel, respectively).
We provide more details in the appendix.

We are not the first to study the $\rk$-context-lumpable stochastic bandit problem.
The earliest work is probably that of \cite{maillard2014latent} who named this problem \emph{latent bandits}: they consider the group identity of the context as latent, or missing information. While the paper introduces this problem, the main theoretical results are derived for the case when the reward distributions given the group information are known to the learner. Further, the results derived are instance dependent. Altogether, while this paper makes interesting contributions, it does not help in answering our questions.%
\footnote{
Somewhat confusingly, \citet{HKZCAB-2020} \emph{define} latent bandit problems differently from
the definition given by \cite{maillard2014latent}: They assume that in addition to the context, a latent (unobserved) variable, influences the rewards, however, they also assume that the distributions of the reward given an action, context and latent state are all known. Their results therefore can not help us in answering our question.
}

The earlier work of \citet{gentile2014online} %
considered another generalization of our problem where in round $t=1,2,\dots$
the learner gets an action set $\{x_{t,1},\dots,x_{t,\na_t}\}\subset \R^d$ and the mean payoff  of choosing action $1\le j \le \na_t$ given the past and the current context $\cx_t$ is $x_{t,j}^\top \theta_{\ca(\cx_t)}$ with $\theta_1,\dots,\theta_{\rk}\in \R^d$ unknown.
When specializing this to our setting, their result depends on the minimum separation $\min_{i,j: \ca(i)\ne \ca(j)} \norm{\rx(i,\cdot)-\rx(j,\cdot)}$, which, according to the authors may be an artifact of their analysis. An extra assumption, which the authors suggest could be removed, is that the distribution of the contexts is uniform (as we shall see, removing this assumption is considerable work). Further, as the authors note, the worst-case for their algorithm is when the contexts are equipartitioned, in which case their bound reduces to the naive bound that one gets for non-lumpable problems.
\citet{li2016collaborative} considers the variation of the problem when the set of arms (and their feature vectors) are fixed and is known before learning begins and separation is defined with respect to this fixed set of arms (and hence could be larger than before). From our perspective, their theoretical results have the same weaknesses as the previous paper.
\citet{gentile17a} further generalizes these previous works to the case when the set of contexts is not fixed a priori. However, the previously mentioned weaknesses of the results remain. %

Very recently, \citet{pal2023optimal} studied the same context-lumpable bandit problem. The authors use similar algorithmic ideas together with an offline matrix completion oracle and obtain $\widetilde{O}(\sqrt{(S+K)T})$ regret when $r$ is constant, under additional incoherence assumptions on the reward matrix.
As a comparison, our algorithm does not require any matrix completion subroutines and does not require those incoherence assumptions.

Our problem is also a special case of learning in \emph{lumpable Markov Decision Processes}. In such Markov Decision Processes (MDPs)
the states can be partitioned such that states within a group of a partition behave identically, both in terms of the transitions and the rewards received \citep[e.g.,][]{RenKro02}.%
\footnote{Lumpability was first introduced first by \citet{kemeny1976finite} in the context of Markov chains.}
We put a survey of this line of work in \cref{app: MDP}.

In summary, although the problem is extensively studied in the literature, we are not aware of any previous results that have implications for the existence of the minimax regret bounds in terms of $\otilde{\sqrt{\poly{(\rk)(\nc+\na)T}}}$, which we believe is a fundamental question in this area.

\section{Notation and problem definition}\label{sec:defs}

For an positive integer $n$, we use $[n]$ to denote the set $\{1,\dots,n\}$.
Further, for a finite set $U$, $\Delta(U)$ denotes the set of probability distributions over $U$ and $\unif{U}$ denotes the uniform distribution over $U$. %
The set of real numbers is denoted by $\R$.

As introduced earlier, we consider context-lumpable bandits with $\nc$ contexts and $\na$ actions such that the $\nc$ contexts can be lumped into $\rk$ groups so that the mean payoff $\rx(\cx,\ac)$ given that
 action $\ac\in [\na]$ is used while the context is
$\cx\in [\nc]$  depends only on the identity of the group that $\cx$ belongs to and the action $\ac$.
That is, for some $\ca:[\nc] \to [\rk]$ map
\begin{align*}
\rx( \cx,\ac ) = \rx( \cx', \ac )
\qquad \text{for any } \ac\in [\na] ,\cx,\cx'\in [\nc] \text{ such that } \ca(\cx)=\ca(\cx')\,.
\end{align*}
Neither $\rx$, nor $\ca$ is known to the learner who interacts with the bandit instance in the following way:
In rounds $t=1,2,\dots$ the learner first observes a context $\cx_t\in [\nc]$ randomly chosen from a fixed, context distribution $\nu \in \Delta([\nc])$,  independently of the past. The distribution $\nu$ is also unknown to the learner.
Next, the learner chooses an action $\ac_t\in [\na]$ to receive a reward of
$
\rw_t = \rx( \cx_t,\ac_t ) + \eta_t\,,
$
where $\eta_t$ is $1$-subgaussian given the past: For any $\lambda \in \R$,
\begin{align*}
\EE{ \exp(\lambda \eta_t ) \mid \cx_1,\ac_1,\dots,\cx_t,\ac_t } \le \exp( \lambda^2/2)\, \text{almost surely (a.s.)}.
\end{align*}
As described earlier, we are interested in two problem settings.
In the \emph{PAC setting}, the learner is given a target suboptimality level $\epsilon>0$ and a target confidence $\delta$
The learner is then tasked to discover a policy $\pi:[\nc] \to [\na]$ that is $\epsilon$-optimal (cf. \cref{eq:sdef}), with probability $\delta$.
In this setting, in addition to deciding what actions to choose, the learner also needs to decide when to stop collecting data and when it stops it has to return a map $\hat \pi:[\nc] \to [\na]$. If $T$ is the (random) time when the learner stops, then the efficiency of the learner is characterized by $\EE{T}$. %
In the \emph{online setting}, the goal of the learner is to keep its regret (cf. \cref{eq:rdef}) low.
In \cref{sec:pac}, we will consider the PAC setting, while the online setting will be considered in \cref{sec:regret}.

In the remainder of the paper, we will also need some further notation.
The map $\ca$ induces a partitioning
$\partition^\star=\{\cls(1),\dots,\cls(\rk)\}$
 of $[\nc]$ in the natural way:
\[
    \cls(\cl)=\left\{\cx:\ca(\cx)=\cl\right\}.
\]
We call each subset $\cls\in\partition^\star$ a \emph{block} and call $\cls(\cl)$ block $\cl$ for every $\cl\in[\rk]$.
We also define block reward $\mu(\cl,\ac)=\rx(\cx,\ac)$ of block $\cl$ for every context $\cx\in\cls(\cl)$ and arm $\ac\in[\na]$.
Finally, we define the block distribution $\db\in\Delta(\rk)$ so that $\db(\cl)=\sum_{\cx\in\cls(\cl)}\dc(\cx)$.

\section{Near-optimal PAC Learning in Context-lumpable Bandits}
\label{sec:pac}

In this section, we present  an algorithm for PAC learning  in context-lumpable bandits, and prove that its  sample complexity guarantee  matches the  minimax rate up to a logarithmic factor.   

\subsection{Special case: almost-uniform context distribution}
To better illustrate the key ideas behind the algorithm design, we first consider a special case of almost-uniform context distribution. Formally, we assume $\nu(i)\in[1/(8S),8/S]$ for all $i\in[S]$ throughout this subsection. 
The pseudocode of our  algorithm is provided in Algorithm \ref{alg:pac-unif-context}, which consists of three main modules. Below we elaborate on each of them separately.

\paragraph{Data collection (\cref{line:pac-data-1}-\ref{line:pac-data-2} and \cref{alg:pac-unif-context-data}).}  The high-level goal of this step is to collect a sufficient amount of 
 data for every block/action pair so that in later steps we have sufficient information to select a near-optimal arm for each block. One naive approach is to try every action on every context for a sufficient amount of time (e.g., $\otilde{1/\epsilon^2}$).
 However, this will cause an undesired factor of $SK$ in the sample complexity. To address this issue, note that contexts from the same block share the same reward function, which means that for every block/action pair  $(b,j)\in[r]\times[K]$, we only need to sufficiently try action $j$ on a single context $i$ from block $b$ (i.e., $i\in\cB(b)$). 
 However, the block structure is unknown a priori. 
 We circumvent this problem by  assigning a random action $\psi(i)$ to each context $i$ for them to explore (\cref{line:pac-rnd-init}). And after action $\psi(i)$ has been tried on context $i$ sufficiently many times, we  update  $\psi(i)$ by reassigning context $i$ with a new random action (\cref{line:pac-rnd-re1}-\ref{line:pac-rnd-re3}).

However, there is still one key problem unresolved yet:  how many samples to use for estimating  each $A(i,\psi(i))$  before reassigning context $i$ with another random action, given a fixed budget of total sample complexity.  
On one hand, if we only try each $(i,\psi(i))$ pair a very few numbers of times before switching, then we can potentially explore a lot of different actions for each block but the accuracy of the reward estimate could be too low to identify a near-optimal policy. 
On the other hand, estimating each $A(i,\psi(i))$ with a huge number of samples resolves the poor accuracy issue but  could potentially cause  the problem of under-exploration, especially for those blocks consisting of very few contexts. 
In the extreme case, if a block only contains a single context, then our total budget may only afford to test a single action on that context (block).

To address this issue, we choose different  levels of  accuracy for different blocks adaptively depending on their block size. Specifically, contexts from larger blocks can afford to try each random action before switching for a larger number of times to achieve higher estimate accuracy because larger blocks consist of more contexts, which means that the average number of random actions in each context inside a larger block needs to try is smaller. 
And for smaller blocks, the case is the opposite. 
Finally, we remark that the above scheme of using adaptive estimate  accuracy  can be implemented without any prior knowledge of the block structure. We only need to iterate over different accuracy levels using an exponential schedule (\cref{line:pac-data-1}), and each block will be sufficiently explored at its appropriate accuracy level.
Specifically, let $N=\lceil\log(1/\epsilon^2)\rceil$, and consider accuracy levels $n\in [N]$. For accuracy level $n$, we collect a dataset $\cD_n$ by adding a context-action pair after the pair is played for $2^n$ timesteps (\cref{line:pac-rnd-re2}). 

\paragraph{Screening optimal action candidates (\cref{line:pac-screen-1}-\ref{line:pac-screen-2}).} Equipped with the dataset collected from Step 1, we want to identify a subset $\cA$ of $[K]$ so that (i) for every block $b\in[r]$, $\cA$ contains a near-optimal action of block $b$, and 
(ii) the size of $\cA$ is as small as possible.
We construct such $\cA$ in an iterative way. 
For each accuracy level $n\in[N]$, we first compute the context-action pair $(i^\star,j^\star)$ with the highest reward estimate in $\cD_n$ (\cref{line:pac-opt-1}). 
Intuitively, as $(i^\star,j^\star)$ has been tried for $2^n$ timesteps in constructing $\cD_n$, it is natural to guess that  $j^\star$ could potentially be a $\otilde{2^{-n/2}}$-optimal action for  certain blocks so we add it into optimal action candidate set $\cA$ (\cref{line:pac-opt-2}). To identify the contexts from those blocks, we sample new data to estimate the reward of $j^\star$ on each context $i\in[S]$.
This can be done by calling \cref{alg:pac-unif-context-data} and setting the exploration set $\mathcal{K}$ to be $j^\star$ (\cref{line: pac test j*}).
If a context $i$ can achieve reward close enough to $\widehat A_n(i^\star,j^\star)$ at action $j^\star$, then we peel off every $(i,j)\in\cD_n$ (\cref{line:pac-screen-2}).
This is because we have found $j^\star$ as an optimal action candidate for $i$ at this accuracy level and don't need to consider other $j$.
We repeat the above process, and add a new arm to $\cA$ in each iteration, until $\cD_n$ becomes empty. 

\paragraph{Solving the simplified context-lumpable bandit (\cref{line:step3}).}
After obtaining the optimal action candidate set $\cA$, we can discard all actions not in $\cA$ and solve a simplified context-lumpable bandit problem by replacing the original action set $[K]$ with $\cA$. Note that $\otilde{S|\mathcal{W}|/\epsilon^2}$ samples suffice for learning an $\epsilon$-optimal policy for this simplified problem. For example,  we can directly try each action $j\in\cA$  on each context $i\in[S]$ for $\widetilde{\Theta}(1/\epsilon^2)$ times, and define $\piout(i)\in\argmax_{j\in\cA} \bar{A}(i,j)$ where $\bar{A}(i,j)$ is the empirical estimate of $A(i,j)$ based on  $\widetilde{\Theta}(1/\epsilon^2)$ samples.

\begin{algorithm}[!ht]\caption{Algorithm for Almost-uniform Context Distribution}\label{alg:pac-unif-context}
\SetAlgoVlined
\nl Let $t$ denote the current time and initialize $N\leftarrow\lceil\log(1/\epsilon^2)\rceil$, $\mathcal{W}\leftarrow\emptyset$, $\lg\leftarrow\Lg$\\
\nl Define $\mathcal{K}(i)\leftarrow[K]$ for $i\in[S]$, $\nes\leftarrow{r(\nc+\na)\lg/\epsilon^2}$\\
{\color{blue} Step 1. Data collection}\\
\nl\label{line:pac-data-1}\For{accuracy level  $n=1,\ldots,N$}{
\nl\label{line:pac-data-2}
Execute \cref{alg:pac-unif-context-data} with input $\nes$, $n$, $\mathcal{K}$, and receive $\mathcal{D}_n$ and $\widehat{A}_n$
}
{\color{blue} Step 2. Screening optimal action candidates}\\
\nl\label{line:pac-screen-1}\For{accuracy level  $n=1,\ldots,N$}{
\nl \While{$\mathcal{D}_n\neq\emptyset$}{
\nl \label{line:pac-opt-1}Compute 
 $(i^\star,j^\star) \leftarrow \argmax_{(i,j)\in\mathcal{D}_n}\widehat A_n(i,j) $\\
\nl \label{line:pac-opt-2}Update optimal action candidates $\mathcal{W}\leftarrow\mathcal{W}\bigcup\{j^\star\}$\\
\nl Update $\mathcal{K}(i)\leftarrow\{j^\star\}$ for $i\in[S]$ and $L'\leftarrow 8\lg 2^nS$\\
\nl\label{line: pac test j*}Execute \cref{alg:pac-unif-context-data} with input $L'$, $n$, $\mathcal{K}$, and reassign output to $\widetilde{\mathcal{D}}_n$ and $\widetilde{A}_n$\\
 \nl\label{line:pac-screen-2}Shrink $\mathcal{D}_n \leftarrow \{(i,j)\in\mathcal{D}_n:~|\widetilde A_n(i,j^\star)-  \widehat A_n(i^\star,j^\star)|\ge \sqrt{\frac{\lg}{2^n}}\}$ %
}
}
{\color{blue} Step 3. Solving the simplified context-lumpable bandit}
\\
\nl \label{line:step3} Use $\frac{4S|\mathcal{W}|}{\epsilon^2} \log\frac{SK}{\delta}$ samples  to find $\piout$ s.t. $A(i,\piout(i)) \ge \max_{j\in\mathcal{W}} A(i,j) -\epsilon $ for all $i\in[S]$ \\
\nl \textbf{Output} $\piout$
\end{algorithm}

\begin{algorithm}[!ht]\caption{Data Collection}\label{alg:pac-unif-context-data}
\setcounter{AlgoLine}{0}
\SetAlgoVlined
\nl \textbf{Input} $\nes$, $n$, $\mathcal{K}$\\
\nl Let $t$ denote the current time and {initialize  $\mathcal{D}_n \leftarrow\emptyset$}\\
    \nl\For{context $\cx\in[\nc]$}{
        \nl\label{line:pac-rnd-init} Assign $\psi_t(\cx)$ to be an arm drawn from $\unif{\mathcal{K}(i)}$}
    \nl\label{line:pac-for}\For{$\nes$ timesteps}{
        \nl Receive context $\cx_{t}$, play arm $\ac_t=\psi(\cx_t)$, and receive reward $\rw_t$\\
    \nl Preset $\psi_{t+1}(i)=\psi_{t}(i)$ for every $i\in[S]$\\
    \nl  \label{line:pac-rnd-re1}\If{$(\sum_{\tau=1}^t \one(i_\tau=i_t))\%2^n=0$}{
     \nl\label{line:pac-rnd-re2}$\mathcal{D}_n\leftarrow \mathcal{D}_n \bigcup\{(i_t,\psi(i_t))\}$ and  $\widehat{A}_n(i_t,\psi(i_t))\leftarrow\frac{\sum_{\tau \le t}\rw_\tau\one[\cx_\tau=i_t,\ac_\tau=\psi(i_t)]}{\sum_{\tau\le t}\one[\cx_\tau=i_t,\ac_\tau=\psi(i_t)]}$\\
     \nl \label{line:pac-rnd-re3} Reassign $\psi_{t+1}(i_t)\sim\unif{\mathcal{K}(i_t)}$;
     }
    }
\nl \textbf{Output} $\mathcal{D}_n$ and $\widehat{A}_n$ %
\end{algorithm}

Now we present the theoretical guarantee for  Algorithm \ref{alg:pac-unif-context}. The proof and exact constants in the bound can be found in Appendix \ref{app:pac-proof}. 

\begin{theorem}\label{thm:pac}  Let $\delta\in (0,1)$ and assume $\nu(i)\in[1/(8S),8/S]$ for all $i\in[S]$. \cref{alg:pac-unif-context}  outputs an $\otilde{\epsilon}$-optimal policy 
 within $\otilde{r(S+K)\log(1/\delta)/\epsilon^2}$ samples with probability at least $1-\delta$. %
\end{theorem}

\subsection{Extension to general context distribution}

\begin{algorithm}[!ht]\caption{Algorithm for General Context Distribution}\label{alg:pac-general-context} 
\setcounter{AlgoLine}{0}
\SetAlgoVlined
\nl Let $J=\frac{4S}{\epsilon}\log(S/\delta)$, $L=\lceil\log(S/\epsilon)\rceil$, $N=524\left(\log\frac{rSK}{\delta}\right) \cdot \left(1+2\log\frac{1}{\epsilon}\right)^2\cdot \frac{r(S+K)}{\epsilon^2}$ \\
\nl Estimate the context distribution by sampling $J$ contexts, and denote the estimate by $\hat\nu$ \\
\nl Split the context set into $L$ disjoint subsets $\{\mathcal{X}_l\}_{l\in[L]}$ where\vspace{-3mm}
$$\vspace{-2mm}
\mathcal{X}_l\leftarrow\begin{cases}
\{i\in[S]:~\hat\nu(i)\in (2^{-l-1},2^{-l}] \}, \quad & l\in [0: L-1]\\
\{i\in[S]:~\hat\nu(i)\in [0,2^{-l}] \}, & l=L
\end{cases}\vspace{-3mm}
$$
\nl \For{$l\in[L-1]$}{
\nl\label{line:learn pi l}Execute Algorithm \ref{alg:pac-unif-context} to learn policy $\pi_l$ for subset $\mathcal{X}_l$ from $N$ time steps
}
\nl \textbf{Output} $\piout$ such that $\piout$ equals to $\pi_l$ over $\mathcal{X}_l$ for $l\in [0: L-1]$, and arbitrary over $\mathcal{X}_L$
\end{algorithm}

In this subsection, we show how to generalize Algorithm \ref{alg:pac-unif-context} to handle general context distributions. We present the pseudo-code in Algorithm \ref{alg:pac-general-context}. 
The algorithm consists of two key steps. In the first step, we use $J=\otilde{S/\epsilon}$ samples to obtain an empirical estimate of the context distribution, denoted by $\hat\nu$. Then we  divide the context set into many disjoint subsets $\{\cX_l\}_{l\in [0:L]}$ such that inside each subset $\cX_l$, the conditional empirical context distribution is almost uniform. 
As a result, we can invoke Algorithm \ref{alg:pac-unif-context} to find a near-optimal policy $\pi_l$ for each subset $\cX_l$ (\cref{line:learn pi l}).
It requires $\otilde{r(\nc+\na)/\epsilon^2}$ time steps for every $\ell$ but we only use samples where contexts are from $\cX_l$ and ignore the rest.
Finally, we glue all $\pi_l$ together to get a policy $\piout$ that is near-optimal for the original problem. Formally, we have the following theoretical guarantee for Algorithm \ref{alg:pac-general-context}. The proof and exact constants are in \cref{app:pac-proof}.

\begin{theorem}\label{thm:pac-general}
Let $\delta\in (0,1)$. Algorithm \ref{alg:pac-general-context} outputs an $\otilde{\epsilon}$-optimal policy within $\otilde{r(S+K)\log(1/\delta)/\epsilon^2}$ samples with probability at least $1-\delta$.
\end{theorem}

Note that we can always learn an $\epsilon$-optimal policy for any context-lumpable bandit within $\otilde{SK/\epsilon^2}$ samples by invoking any existing near-optimal algorithm for finite contextual bandits. 
As a result, by combining Theorem \ref{thm:pac-general} with the $\otilde{SK/\epsilon^2}$ upper bound,   we obtain that  
$\otilde{\min\{r(S+K),SK\}/\epsilon^2}$ samples suffice for learning any context-lumpable bandit, which according to the following theorem is minimax optimal up to a logarithmic factor.
\begin{theorem}\label{thm:pac lower-bound}
Learning an $\epsilon$-optimal policy for a context-lumpable bandit with probability no smaller than $1/2$ requires at least  
$\Omega(\min\{r(S+K),SK\}/\epsilon^2)$ samples in the worst case.
\end{theorem}

\section{Regret Minimization in Context-lumpable Bandits}\label{sec:regret}

In this section, we extend the idea from the PAC setting to the online setting.
To better introduce the key ideas, we first consider a special case when both context and block distributions are uniform (\cref{sec: regret multiple blocks}).
Then we consider the most general case in \cref{sec: regret non-uniform}.

\subsection{Special Case: Uniform Context and Block Distribution}\label{sec: regret multiple blocks}
In this section, we assume that distributions $\dc$ and $\db$ are uniform, and thus, $\ca$ evenly splits the contexts into $\rk$ blocks so that there are $\nc/\rk$ contexts in each block and every context appears with the same probability at each timestep.
We will relax these assumptions and consider the general case in the next subsection. %

For this case, we introduce \cref{alg: multiple blocks}, which uses phased elimination in combination with a clustering procedure. The algorithm runs in phases $h=1,2,\dots$ that are specified by error tolerance parameter $\epsilon_h=2^{-h/2}$.
Like phased elimination algorithms for multi-armed bandits, we need to ensure at phase $h$ actions the algorithms play are all $\otilde{\epsilon_h}$-optimal. 
Thus, at the end of each phase $h$, we eliminate all actions that are not $\otilde{\epsilon_h}$-optimal.
However, the set of $\otilde{\epsilon_h}$-optimal actions of each block is different. 
Therefore, we also perform \emph{clustering} on contexts and perform elimination for each subset of the partition.
Specifically, we maintain a partition of contexts $\partition_h$ at each phase $h$ and initialize $\partition_1=\{[\nc]\}$.
For each cluster $\cs\in\partition_h$, we maintain a set of good arms $\good_h(\cs)$, which we will prove is $\otilde{\epsilon_h}$-optimal for contexts in the cluster with high probability. %

We use \cref{alg:pac-unif-context-data} to collect data similar to \cref{alg:pac-unif-context} (\cref{line: regret uniform data}).
At phase $h$, we try to estimate mean reward up to error $\otilde{\epsilon_h}$ with probability $1-\delta_h$.
The difference is that we assume $\db$ is uniform, so we don't need different accuracy levels $n$, which will be required for the algorithm that handles the general case.
Also, for every context $i\in\cs$, we only explore arms \emph{good for now}, that is, in $\good_h(\cs)$ instead of exploring all the arms $[K]$.
This reflects that in the online setting, we need to minimize regret and cannot afford to explore bad arms too many times.

Based on the data we collect during the exploration stage, we then check if there is a large gap across contexts in the same subset for any arms (\cref{line: detect large gaps - multiple}).
A large gap suggests that (i) the subset contains at least two blocks (ii) the mean reward of the arm is significantly different in these blocks and we can use this arm to partition the contexts by running \cref{alg: the cluster algorithm - multiple} (clustering stage).
We repeatedly do clustering (\cref{line: separable - multiple}) and split heterogeneous subsets (\cref{line:init good}) until we cannot find a large gap within the same subset.
If no large gap is detected, then each arm has similar mean rewards (up to error $\otilde{\epsilon_h}$) across all blocks in the same subset. Then we eliminate arms that are significantly worse than the empirical best arm (elimination stage) from $\good_h(\cs)$ for every subset $\cs$ and start a new episode (\cref{line: perform arm elimination - multiple}).

\cref{alg: the cluster algorithm - multiple} plays $\ac$ for each context $i\in\cs$ for $\otilde{{1}/{\epsilon^2}}$ rounds and calculates empirical means of arm $j$ for each context in $\cs$. It then sorts the contexts by the empirical means and performs clustering (\cref{line: sort means}). Specifically, the algorithm enumerates contexts in descending order of empirical means and splits contexts until a large gap between consecutive values is detected (\cref{line: clustering large gap}).
As we call \cref{alg: the cluster algorithm - multiple} only if a large gap is detected, we prove that in the appendix it correctly separates the subset into at least two parts without putting any contexts in the same block into different parts by setting $\epsilon'=\epsilon_h/r$. 

\begin{algorithm}[!ht]\caption{Algorithm for Uniform Block Distribution}\label{alg: multiple blocks}
\setcounter{AlgoLine}{0}
\SetAlgoVlined
\nl
Initialize $\partition_1\leftarrow\{[\nc]\}$, $\good_1(\cs)\leftarrow[\na]$ for $\cs\in\partition_1$\\
\nl Let $t$ denote the current time\\
\nl\For{phase $h=1,2,\dots,$}{
    \nl Let $\epsilon_h\leftarrow2^{-h/2}$, $\delta_h\leftarrow \epsilon_h^2/(r^3SK)$, $\lg_h\leftarrow 64\log(rSK/\delta_h)$, $\epst_h\leftarrow\sqrt{\lg_h}\cdot\epsilon_h$\\[0.1cm]
    {\color{blue} Step 1. Data collection}\\[0.1cm]
    \nl Define $\nes_h\leftarrow\tfrac{r(S+K)\lg_h}{\epsilon_h^2}$,\quad $n_h\leftarrow \log(\tfrac{1}{\epsilon_h^2})$,\quad $\mathcal{K}_h(i)\leftarrow\good_h(\cs)$, $\cs\in\partition_h$, $\cs\ni i$, $\forall i\in[S]$\\
    \nl\label{line: regret uniform data}Execute \cref{alg:pac-unif-context-data} with input $\nes_h$, $n_h$, $\mathcal{K}_h$, and receive $\mathcal{D}_h$ and $\widehat{A}_{h}$\\[0.05cm]
    {\color{blue} Step 2. Test homogeneity and perform clustering on heterogeneous subsets}\\ [0.1cm]
    \nl\label{line: detect large gaps - multiple}\While{$\exists\,\cs\in\partition_h, \overline{i},\underline{i}\in\cs,\ac\in[K]$ such that $\widehat{A}_{h}(\overline{i},\ac)-\widehat{A}_{h}(\underline{i},\ac)\ge \epst_h$}{
        \ \\
        \nl Define $\epsilon'\leftarrow \frac{\epsilon_h}{4\rk}$, $\delta'\leftarrow\frac{\delta_h}{\rk}$, $\mathcal{K}(i)\leftarrow\{j\}$ if $i\in\cs$ else $\good_h(\cs)$ for  $\cs\in\partition_h$, $\cs\ni i$\\[0.2cm]
        \nl\label{line: separable - multiple}Execute {\cref{alg: the cluster algorithm - multiple}} with input $\epsilon'$, $\delta'$, $\mathcal{K}$, $\cs$, and $\ac$, and get $\partition$, a partition of $\cs$ \\[0.2cm]
        \nl\label{line:init good}Initialize $\good_h(\cs')\leftarrow\good_h(\cs)$, $\forall\cs'\in\partition$ and update $\partition_h\leftarrow(\partition\cup\partition_h)\backslash\{\cs\}$ %
    }
    {\color{blue} Step 3. Eliminate suboptimal actions in each subset}\\[0.1cm]
    \nl $\partition_{h+1}\leftarrow\partition_h$\\[0.1cm]
    \nl\For{$\cs\in\partition_{h+1}$}{
        \nl Calculate $\muu_h(\cs,\ac)\leftarrow\max_{\cx:\cx\in\cs,(i,j)\in\mathcal{D}_h}\rxh_h(\cx,\ac)$, $\forall j\in\good_h(\cs)$\\
        \nl\label{line: perform arm elimination - multiple}Update $\good_{h+1}(\cs)\leftarrow\left\{\ac:\ac\in\good_h(\cs),~\max_{j'}\muu_h(\cs,\ac')-\muu_h(\cs,\ac)\le 2\epst_h\right\}$
    }
}
\end{algorithm}

\begin{algorithm}\caption{Split Contexts Into Multiple Blocks}\label{alg: the cluster algorithm - multiple}
\setcounter{AlgoLine}{0}
\SetAlgoVlined
\nl\textbf{Input}: error $\epsilon'$, confidence $\delta'$, exploration sets $\mathcal{K}$, subset $\cs$, separating arm $\ac$\\
\nl Initialize $\lg\leftarrow 64\log(S/\delta')$, $L'\leftarrow{{\nc\lg}/{\epsilon'^2}}$, $n'\leftarrow\lceil1/\epsilon'^2\rceil$\\
\nl Execute \cref{alg:pac-unif-context-data} with input $L'$, $n'$, $\mathcal{K}$, and reassign output to ${\mathcal{D}}$ and $\rxh$\\
\nl\label{line: sort means}Sort contexts in $\cs$ and label them as $\cx_1,\dots,\cx_{|\cs|}$ so that $\rxh(\cx_1,j)\ge \rxh(\cx_2,j)\ge\cdots\ge\rxh(\cx_{|\cs|},j)$\\
\nl Initialize $\partition_1\leftarrow\{\cx_1\}$, $\cl\leftarrow 1$\\
\nl\For{$k=2,3,\dots,|\cs|$}
{
\nl\label{line: clustering large gap}    \If{$\rxh({\cx_{k-1},j})-\rxh({\cx_{k}},j)\ge\sqrt{\lg}\cdot \epsilon'$}{
        \nl Update $\cl\leftarrow \cl+1$ and initialize $\partition_{\cl}\leftarrow\{\}$
    }
    \nl $\partition_{\cl}\leftarrow \partition_{\cl}\cup\{\cx_k\}$
}
\nl \textbf{Output} $\{\partition_1,\dots,\partition_\cl\}$
\end{algorithm}
Similar to the analysis of other phased elimination algorithms, we have to show that in a phase specified by error level $\epsilon_h$, with high probability, (i) the optimal arm is not eliminated and (ii) all $\widetilde{\omega}(\epsilon_h)$-suboptimal arms are eliminated, that is, all arms in $\good_h(\cs)$ for all $\cs$ are  $\otilde{\epsilon_h}$-optimal.
We show the final regret here and defer the details to \cref{app: regret multiple blocks}

\begin{theorem}\label{thm: regret multiple}
    Under the assumption that context distribution $\dc$ and block distribution $\db$ are uniform, regret of \cref{alg: multiple blocks non-uniform} is bounded as $\rg_T=\otilde{\sqrt{\rk^3(\nc+\na)T}}$.
\end{theorem}
Compared to our PAC result, we get an extra dependency on $\rk$.
This is because we pay $\otilde{S/\epsilon'^2}=\otilde{{\rk^2S}/{\epsilon_h^2}}$ samples to do clustering instead of $\otilde{{1}/{\epsilon_h^2}}$ in order to ensure a ``perfect'' partition, that is, we never separate contexts in the same block with high probability.
This is crucial for our phase-elimination algorithm as we may call \cref{alg: the cluster algorithm - multiple} in different phases.
We left getting better than $\otilde{\sqrt{\rk^3(\nc+\na)T}}$ regret upper bounds or better than ${\Omega}({\sqrt{r(\nc+\na)T}})$ regret lower bounds (even for this uniform special case) as an important future direction.

\subsection{Non-uniform Context and Block Distribution}\label{sec: regret non-uniform}
Similar to the PAC learning setting, we can use \cref{alg:pac-general-context} to reduce general context distributions to (nearly) uniform ones.
We provide more details in \cref{app: context non-uniform}.
As a result, without loss of generality, we may assume $\dc$ is almost-uniform, and we focus on how to handle non-uniform block distribution $\db$.
In the extreme, there may only be one context for some blocks, which becomes challenging to estimate their mean rewards. 

For this case, we introduce \cref{alg: multiple blocks non-uniform}.
Intuitively, based on \cref{alg: multiple blocks}, we can further employ different accuracy levels as \cref{alg:pac-unif-context}.
However, as our goal becomes minimizing regret, it is difficult to control regret for smaller $n$ and larger blocks. 
Specifically, for smaller $n$, we explore more actions (with fewer samples) for each context in a single phase.
Since fewer samples are used, we may unavoidably play suboptimal actions and suffer large regret.
This becomes worse for a large block as more contexts in the block suffer large regret. 
We note that this is not a problem in the PAC setting because we only need to control sample complexity. 
\begin{algorithm}[!ht]\caption{Algorithm for Non-uniform Block Distribution}\label{alg: multiple blocks non-uniform}
\setcounter{AlgoLine}{0}
\SetAlgoVlined
\nl
Initialize $\partition_1\leftarrow\{[\nc]\}$, $\good_{1,1}(\cs)\leftarrow[\na]$, $\cs\in\partition_1$\\
\nl Let $t$ denote the current time\\
\nl\For{phase $h=1,2,\dots,$}{
    \nl Let $\epsilon_h\leftarrow2^{-h/2}$, $N_h\leftarrow h$, $\delta_h\leftarrow \epsilon_h^2/(r^3SK)$, $\lg_h\leftarrow 128\log(rSKN_h/\delta_h)$\\[0.1cm]
    {\color{blue} Step 1. Data collection}\\
    \nl \For{accuracy level $n=1,\ldots,N_h$}{
    \ \\
        \nl Define $\nes_{h,n}\leftarrow{r(S+K)\lg_h 2^{(n+h)/2}}$, $\mathcal{K}_{h,n}(i)\leftarrow\good_{h,n}(\cs)$, $\cs\in\partition_h$, $\cs\ni i$, $\forall i\in[S]$\\
        \nl Execute \cref{alg:pac-unif-context-data} with input $\nes_h$, $n$, $\mathcal{K}_{h,n}$, and receive $\mathcal{D}_{h,n}$ and $\widehat{A}_{h,n}$%
    }
    {\color{blue} Step 2. Test homogeneity and perform clustering on heterogeneous subsets}\\ [0.1cm]
    \nl\label{line: detect large gaps - non-uniform}\While{$\exists\,\cs\in\partition_h, \overline{i},\underline{i}\in\cs,\ac\in[K]$, $n\in[N_h]$ such that $\widehat{A}_{h,n}(\overline{i},\ac)-\widehat{A}_{h,n}(\underline{i},\ac)\ge \sqrt{\frac{\lg_h}{2^{n}}}$}{
        \ \\
        \nl Define $\epsilon'\leftarrow \frac{\epsilon_h}{4\rk}$, $\delta'\leftarrow\frac{\delta_h}{\rk}$, $\mathcal{K}(i)\leftarrow\{j\}$ if $i\in\cs$ else $\good_{h,n}(\cs)$ for  $\cs\in\partition_h$, $\cs\ni i$\\[0.2cm]
        \nl Execute {\cref{alg: the cluster algorithm - multiple}} with input $\epsilon'$, $\delta'$, $\mathcal{K}$, $\cs$, and $\ac$, and get $\partition$, a partition of $\cs$ \\[0.2cm]
        \nl Initialize $\good_{h,n}(\cs')\leftarrow\good_{h,n}(\cs)$, $\forall\cs'\in\partition$ and update $\partition_h\leftarrow(\partition\cup\partition_h)\backslash\{\cs\}$ %
    }
    {\color{blue} Step 3. Eliminate suboptimal actions in each subset}\\[0.1cm]
    \nl $\partition_{h+1}\leftarrow\partition_h$, $\good_{1,1}(\cs)\leftarrow[\na]$, $\cs\in\partition_{h+1}$\\[0.1cm]
    \nl \For{accuracy level $\lv=2,\dots,N_h$}{
        \nl\For{$\cs\in\partition_{h+1}$}{
            \nl Calculate $\muu_{h,n}(\cs,\ac)\leftarrow\max_{\cx:\cx\in\cs,(i,j)\in\mathcal{D}_{h,n}}\rxh_{h,n}(\cx,\ac)$, $\forall j\in\good_{h,n}(\cs)$\\
            \nl Let $\mathcal{G}_{h,n}(\cs)\leftarrow\left\{\ac:\ac\in\good_{h,n}(\cs),~\max_{j'}\muu_{h,n}(\cs,\ac')-\muu_{h,n}(\cs,\ac)\le 2\sqrt{\frac{\lg_h}{2^n}}\right\}$\\
            \nl\label{line:inclusion}Update $\good_{h+1,n}(\cs)\leftarrow\good_{h+1,n-1}(\cs)\cap\mathcal{G}_{h,n}(\cs)$
        }
    }
}
\end{algorithm}

We fix the issue by setting different lengths $L$ for different accuracy levels.
Recall in \cref{alg:pac-unif-context}, we use the same length $L=\otilde{r(S+K)/\epsilon^2}$ for every $n$.
Intuitively, since we allow less accurate estimations for smaller $n$, we may use fewer data.
It turns out indeed we can set $\nes={\otilde{r(S+K)2^{(n+h)/2}}}$ for level $n$ at phase $h$, and with more refined analysis, it provides similar guarantees as before. 
More importantly, since smaller $n$ uses (exponentially) fewer samples, the overall regret is well controlled.

In addition to setting the lengths sophisticatedly, we also need to carefully maintain sets of good actions $\good_{h,n}$ not only at each phase $h$ but also at each accuracy level $n$.
The partition of contexts $\partition_h$, however, only evolves with phases and is shared between different levels at the same phase. 
Similar to \cref{alg: multiple blocks}, we also do clustering (Step 2) and elimination (Step 3) but do the procedures in parallel for every accuracy level $\lv$.
In the clustering stage, we use different thresholds to define a ``large gap''. This reflects that $n$ represents different accuracy levels and thus has different widths of confidence intervals. 
For the elimination stage, we enforce the inclusion relation $\good_h(\lv,\cs)\subseteq\good_h(\lv',\cs)$ for $\lv'\le\lv$ (\cref{line:inclusion}), which will be useful in the regret analysis.
We now present the main theorem and put the complete proof in \cref{app: regret non-uniform}.
\begin{theorem}\label{thm: regret non-uniform}
Regret of \cref{alg: multiple blocks non-uniform} is bounded as $\rg_T=\otilde{\sqrt{\rk^3(\nc+\na)T}}$ \;.
\end{theorem}
We remark that \cref{alg: multiple blocks non-uniform} is \emph{anytime}, that is, it doesn't require the time horizon $T$ as input.
However, it does require the number of blocks $r$ as prior knowledge. 
Removing the knowledge of $r$ is an interesting future direction.
One promising idea is to apply a doubling trick on $r$.
Specifically, we have an initial guess $r=\otilde{1}$ and run  \cref{alg: multiple blocks non-uniform}; when $|\partition_h|> r$, we double $r$ and restart the algorithm.
The analysis, though, may be much more complicated.
Finally, we note that when knowing $r$, one can calculate in advance if $\sqrt{\nc\na T}$ is less than $\sqrt{\rk^3(\nc+\na)T}$ and switch to a standard algorithm achieving $\otilde{\sqrt{\nc\na T}}$ in that case.
This modification guarantees the regret is never worse than the standard bound. 

\lowrank{
\section{From Context-lumpable Bandits to Contextual Low-rank Bandits}\label{sec: low-rank}
Finally, we consider the more general \emph{contextual low-rank bandit} problem. %
Specifically, we allow $\rx$ to have rank $\rk$, that is $\rx= U V$ for some $\nc\times \rk$ matrix $U$ and $\rk\times \na$ matrix $V$. Lumpability is a special case in the sense that $U$ is binary where each row has a single nonzero element.

To solve the problem, We show a reduction from contextual low-rank bandits to context-lumpable bandits. Consider an $\alpha$-covering of rows of $U$, and notice that the covering number, $\rr$, is ${1}/{\alpha^\rk}$. The context-lumpable bandits can be seen as $\alpha$-approximate context-lumpable bandits with $\rr$ blocks, where the reward of contexts on the same block differs at most $\alpha$.
Ignoring this misspecification, we may run \cref{alg:pac-unif-context} and \cref{alg: multiple blocks non-uniform} for the PAC and online settings, respectively.
Moreover, it turns out that our algorithms are robust to this misspecification when $\alpha$ is sufficiently small ($O(\epsilon)$ in the PAC setting, for example).
Therefore, the sample complexity and the regret bounds of these algorithms will be in terms of $\rr$ despite having an exponential dependency on $\rk$.
The resulting bounds can still be smaller than the naive ${\nc\na}/{\epsilon^2}$ and $\sqrt{\nc\na T}$ bounds in some scenarios, for example, when $S$ and $K$ are super large.
We put the details in \cref{app: low-rank}.}

\section{Conclusions and Future Directions}

We consider a contextual bandit problem with $S$ contexts and $K$ actions. Under the assumption that the context-action reward matrix has $r\le \min\{S,K\}$ unique rows, we show an algorithm that outputs an $\epsilon$-optimal policy and has the optimal sample complexity of $\widetilde O(r(S+K)/\epsilon^2)$ with high probability. In the regret minimization setting, we show an algorithm whose cumulative regret up to time $T$ is bounded as $\widetilde O(\sqrt{r^3(S+K)T})$.

An immediate next question is whether a regret bound of order $\widetilde O(\sqrt{r(S+K)T})$ is achievable in the regret minimization setting. A second open question is concerned with obtaining a $\widetilde O(\sqrt{\poly(r)(S+K)T})$ regret bound in contextual low-rank bandits. 
Our regret analysis heavily relies on the assumption that contexts arrive in an I.I.D fashion. Extending our results to the setting with adversarial contexts remains another important future direction.

\bibliography{refs}

\begin{thebibliography}{49}
\providecommand{\natexlab}[1]{#1}
\providecommand{\url}[1]{\texttt{#1}}
\expandafter\ifx\csname urlstyle\endcsname\relax
  \providecommand{\doi}[1]{doi: #1}\else
  \providecommand{\doi}{doi: \begingroup \urlstyle{rm}\Url}\fi

\bibitem[Abbasi-Yadkori et~al.(2011)Abbasi-Yadkori, P{\'a}l, and
  Szepesv{\'a}ri]{abbasi2011improved}
Yasin Abbasi-Yadkori, D{\'a}vid P{\'a}l, and Csaba Szepesv{\'a}ri.
\newblock Improved algorithms for linear stochastic bandits.
\newblock \emph{Advances in neural information processing systems}, 24, 2011.

\bibitem[Arora et~al.(2012)Arora, Ge, and Moitra]{arora2012learning}
Sanjeev Arora, Rong Ge, and Ankur Moitra.
\newblock Learning topic models--going beyond svd.
\newblock In \emph{2012 IEEE 53rd annual symposium on foundations of computer
  science}, pages 1--10. IEEE, 2012.

\bibitem[Auer(2002)]{Aue02}
P.~Auer.
\newblock Using confidence bounds for exploitation-exploration trade-offs.
\newblock \emph{Journal of Machine Learning Research}, 3:\penalty0 397--422,
  2002.

\bibitem[Auer et~al.(2002)Auer, Cesa-Bianchi, Freund, and
  Schapire]{auer2002nonstochastic}
Peter Auer, Nicolo Cesa-Bianchi, Yoav Freund, and Robert~E Schapire.
\newblock The nonstochastic multiarmed bandit problem.
\newblock \emph{SIAM journal on computing}, 32\penalty0 (1):\penalty0 48--77,
  2002.

\bibitem[Beygelzimer et~al.(2011)Beygelzimer, Langford, Li, Reyzin, and
  Schapire]{beygelzimer2011contextual}
Alina Beygelzimer, John Langford, Lihong Li, Lev Reyzin, and Robert Schapire.
\newblock Contextual bandit algorithms with supervised learning guarantees.
\newblock In \emph{Proceedings of the Fourteenth International Conference on
  Artificial Intelligence and Statistics}, pages 19--26. JMLR Workshop and
  Conference Proceedings, 2011.

\bibitem[Cella et~al.(2020)Cella, Lazaric, and Pontil]{CLP-2020}
Leonardo Cella, Alessandro Lazaric, and Massimiliano Pontil.
\newblock Meta-learning with stochastic linear bandits.
\newblock \emph{Arxiv}, 2020.

\bibitem[Chen et~al.(2020)Chen, Chi, Fan, Ma, and Yan]{chen2020noisy}
Yuxin Chen, Yuejie Chi, Jianqing Fan, Cong Ma, and Yuling Yan.
\newblock Noisy matrix completion: Understanding statistical guarantees for
  convex relaxation via nonconvex optimization.
\newblock \emph{SIAM journal on optimization}, 30\penalty0 (4):\penalty0
  3098--3121, 2020.

\bibitem[Chu et~al.(2011)Chu, Li, Reyzin, and Schapire]{chu2011contextual}
Wei Chu, Lihong Li, Lev Reyzin, and Robert Schapire.
\newblock Contextual bandits with linear payoff functions.
\newblock In \emph{Proceedings of the Fourteenth International Conference on
  Artificial Intelligence and Statistics}, pages 208--214. JMLR Workshop and
  Conference Proceedings, 2011.

\bibitem[Dann et~al.(2018)Dann, Jiang, Krishnamurthy, Agarwal, Langford, and
  Schapire]{dann2018oracle}
Christoph Dann, Nan Jiang, Akshay Krishnamurthy, Alekh Agarwal, John Langford,
  and Robert~E Schapire.
\newblock On oracle-efficient {PAC RL} with rich observations.
\newblock \emph{Advances in neural information processing systems}, 31, 2018.

\bibitem[Du et~al.(2019)Du, Krishnamurthy, Jiang, Agarwal, Dudik, and
  Langford]{du2019provably}
Simon Du, Akshay Krishnamurthy, Nan Jiang, Alekh Agarwal, Miroslav Dudik, and
  John Langford.
\newblock Provably efficient {RL} with rich observations via latent state
  decoding.
\newblock In \emph{International Conference on Machine Learning}, pages
  1665--1674. PMLR, 2019.

\bibitem[Duan et~al.(2019)Duan, Ke, and Wang]{duan2019state}
Yaqi Duan, Tracy Ke, and Mengdi Wang.
\newblock State aggregation learning from {M}arkov transition data.
\newblock \emph{Advances in Neural Information Processing Systems}, 32, 2019.

\bibitem[Dudik et~al.(2011)Dudik, Hsu, Kale, Karampatziakis, Langford, Reyzin,
  and Zhang]{dudik2011efficient}
Miroslav Dudik, Daniel Hsu, Satyen Kale, Nikos Karampatziakis, John Langford,
  Lev Reyzin, and Tong Zhang.
\newblock Efficient optimal learning for contextual bandits.
\newblock In \emph{Proceedings of the Twenty-Seventh Conference on Uncertainty
  in Artificial Intelligence}, pages 169--178, 2011.

\bibitem[Feng et~al.(2020)Feng, Wang, Yin, Du, and Yang]{feng2020provably}
Fei Feng, Ruosong Wang, Wotao Yin, Simon~S Du, and Lin Yang.
\newblock Provably efficient exploration for reinforcement learning using
  unsupervised learning.
\newblock \emph{Advances in Neural Information Processing Systems},
  33:\penalty0 22492--22504, 2020.

\bibitem[Foster and Rakhlin(2020)]{foster2020beyond}
Dylan Foster and Alexander Rakhlin.
\newblock Beyond ucb: Optimal and efficient contextual bandits with regression
  oracles.
\newblock In \emph{International Conference on Machine Learning}, pages
  3199--3210. PMLR, 2020.

\bibitem[Gentile et~al.(2014)Gentile, Li, and Zappella]{gentile2014online}
Claudio Gentile, Shuai Li, and Giovanni Zappella.
\newblock Online clustering of bandits.
\newblock In \emph{International Conference on Machine Learning}, pages
  757--765. PMLR, 2014.

\bibitem[Gentile et~al.(2017)Gentile, Li, Kar, Karatzoglou, Zappella, and
  Etrue]{gentile17a}
Claudio Gentile, Shuai Li, Purushottam Kar, Alexandros Karatzoglou, Giovanni
  Zappella, and Evans Etrue.
\newblock On context-dependent clustering of bandits.
\newblock In \emph{Proceedings of the 34th International Conference on Machine
  Learning (ICML)}, volume~70 of \emph{Proceedings of Machine Learning
  Research}, pages 1253--1262, 2017.

\bibitem[Hong et~al.(2020)Hong, Kveton, Zaheer, Chow, Ahmed, and
  Boutilier]{HKZCAB-2020}
Joey Hong, Branislav Kveton, Manzil Zaheer, Yinlam Chow, Amr Ahmed, and Craig
  Boutilier.
\newblock Latent bandits revisited.
\newblock In \emph{NeurIPS}, 2020.

\bibitem[Jain and Pal(2022)]{JP-2022}
Prateek Jain and Soumyabrata Pal.
\newblock Online low rank matrix completion, 2022.

\bibitem[Jain et~al.(2013)Jain, Netrapalli, and Sanghavi]{jain2013low}
Prateek Jain, Praneeth Netrapalli, and Sujay Sanghavi.
\newblock Low-rank matrix completion using alternating minimization.
\newblock In \emph{Proceedings of the forty-fifth annual ACM symposium on
  Theory of computing}, pages 665--674, 2013.

\bibitem[Jang et~al.(2021)Jang, Jun, Yun, and Kang]{jang2021improved}
Kyoungseok Jang, Kwang-Sung Jun, Se-Young Yun, and Wanmo Kang.
\newblock Improved regret bounds of bilinear bandits using action space
  analysis.
\newblock In \emph{International Conference on Machine Learning}, pages
  4744--4754. PMLR, 2021.

\bibitem[Jiang et~al.(2017)Jiang, Krishnamurthy, Agarwal, Langford, and
  Schapire]{jiang2017contextual}
Nan Jiang, Akshay Krishnamurthy, Alekh Agarwal, John Langford, and Robert~E
  Schapire.
\newblock Contextual decision processes with low {B}ellman rank are
  {PAC}-learnable.
\newblock In \emph{International Conference on Machine Learning}, pages
  1704--1713. PMLR, 2017.

\bibitem[Jun et~al.(2019)Jun, Willett, Wright, and Nowak]{jun2019bilinear}
Kwang-Sung Jun, Rebecca Willett, Stephen Wright, and Robert Nowak.
\newblock Bilinear bandits with low-rank structure.
\newblock In \emph{International Conference on Machine Learning}, pages
  3163--3172. PMLR, 2019.

\bibitem[Kang et~al.(2022)Kang, Hsieh, and Lee]{kang2022efficient}
Yue Kang, Cho-Jui Hsieh, and Thomas Chun~Man Lee.
\newblock Efficient frameworks for generalized low-rank matrix bandit problems.
\newblock In \emph{Advances in Neural Information Processing Systems}, 2022.

\bibitem[Katariya et~al.(2017)Katariya, Kveton, Szepesvari, Vernade, and
  Wen]{katariya2017stochastic}
Sumeet Katariya, Branislav Kveton, Csaba Szepesvari, Claire Vernade, and Zheng
  Wen.
\newblock Stochastic rank-1 bandits.
\newblock In \emph{Artificial Intelligence and Statistics}, pages 392--401.
  PMLR, 2017.

\bibitem[Kemeny and Snell(1976)]{kemeny1976finite}
John~G Kemeny and James~Laurie Snell.
\newblock \emph{Finite {M}arkov chains}.
\newblock Springer, 1976.

\bibitem[Kveton et~al.(2017)Kveton, Szepesv{\'a}ri, Rao, Wen, Abbasi-Yadkori,
  and Muthukrishnan]{kveton2017stochastic}
Branislav Kveton, Csaba Szepesv{\'a}ri, Anup Rao, Zheng Wen, Yasin
  Abbasi-Yadkori, and S~Muthukrishnan.
\newblock Stochastic low-rank bandits.
\newblock \emph{arXiv preprint arXiv:1712.04644}, 2017.

\bibitem[Kveton et~al.(2020)Kveton, Mladenov, Hsu, Zaheer, Szepesv\'{a}ri, and
  Boutilier]{kveton-2020}
Branislav Kveton, Martin Mladenov, Chih-Wei Hsu, Manzil Zaheer, Csaba
  Szepesv\'{a}ri, and Craig Boutilier.
\newblock Differentiable meta-learning in contextual bandits.
\newblock \emph{arXiv:2006.05094v1}, 2020.

\bibitem[Kveton et~al.(2021)Kveton, Konobeev, Zaheer, wei Hsu, Mladenov,
  Boutilier, and Szepesvari]{KKZHMBS-2021}
Branislav Kveton, Mikhail Konobeev, Manzil Zaheer, Chih wei Hsu, Martin
  Mladenov, Craig Boutilier, and Csaba Szepesvari.
\newblock Meta-{T}hompson sampling.
\newblock \emph{Arxiv}, 2021.

\bibitem[Kwon et~al.(2021)Kwon, Efroni, Caramanis, and Mannor]{kwon2021rl}
Jeongyeol Kwon, Yonathan Efroni, Constantine Caramanis, and Shie Mannor.
\newblock {RL} for latent {MDP}s: Regret guarantees and a lower bound.
\newblock \emph{Advances in Neural Information Processing Systems},
  34:\penalty0 24523--24534, 2021.

\bibitem[Lattimore and Hao(2021)]{lattimore2021bandit}
Tor Lattimore and Botao Hao.
\newblock Bandit phase retrieval.
\newblock \emph{Advances in Neural Information Processing Systems},
  34:\penalty0 18801--18811, 2021.

\bibitem[Lattimore and Szepesv{\'a}ri(2020)]{lattimore2020bandit}
Tor Lattimore and Csaba Szepesv{\'a}ri.
\newblock \emph{Bandit algorithms}.
\newblock Cambridge University Press, 2020.

\bibitem[Li et~al.(2016)Li, Karatzoglou, and Gentile]{li2016collaborative}
Shuai Li, Alexandros Karatzoglou, and Claudio Gentile.
\newblock Collaborative filtering bandits.
\newblock In \emph{Proceedings of the 39th International ACM SIGIR conference
  on Research and Development in Information Retrieval}, pages 539--548, 2016.

\bibitem[Lu et~al.(2021)Lu, Meisami, and Tewari]{lu2021low}
Yangyi Lu, Amirhossein Meisami, and Ambuj Tewari.
\newblock Low-rank generalized linear bandit problems.
\newblock In \emph{International Conference on Artificial Intelligence and
  Statistics}, pages 460--468. PMLR, 2021.

\bibitem[Maillard and Mannor(2014)]{maillard2014latent}
Odalric-Ambrym Maillard and Shie Mannor.
\newblock Latent bandits.
\newblock In \emph{International Conference on Machine Learning}, pages
  136--144. PMLR, 2014.

\bibitem[Misra et~al.(2020)Misra, Henaff, Krishnamurthy, and
  Langford]{misra2020kinematic}
Dipendra Misra, Mikael Henaff, Akshay Krishnamurthy, and John Langford.
\newblock Kinematic state abstraction and provably efficient rich-observation
  reinforcement learning.
\newblock In \emph{International conference on machine learning}, pages
  6961--6971. PMLR, 2020.

\bibitem[Modi et~al.(2021)Modi, Chen, Krishnamurthy, Jiang, and
  Agarwal]{modi2021model}
Aditya Modi, Jinglin Chen, Akshay Krishnamurthy, Nan Jiang, and Alekh Agarwal.
\newblock Model-free representation learning and exploration in low-rank
  {MDP}s.
\newblock \emph{arXiv preprint arXiv:2102.07035}, 2021.

\bibitem[Ni et~al.(2021)Ni, Zhang, Duan, and Wang]{ni2021learning}
Chengzhuo Ni, Anru~R Zhang, Yaqi Duan, and Mengdi Wang.
\newblock Learning good state and action representations via tensor
  decomposition.
\newblock In \emph{2021 IEEE International Symposium on Information Theory
  (ISIT)}, pages 1682--1687. IEEE, 2021.

\bibitem[Pal et~al.(2023)Pal, Suggala, Shanmugam, and Jain]{pal2023optimal}
Soumyabrata Pal, Arun~Sai Suggala, Karthikeyan Shanmugam, and Prateek Jain.
\newblock Optimal algorithms for latent bandits with cluster structure.
\newblock In \emph{International Conference on Artificial Intelligence and
  Statistics}, pages 7540--7577. PMLR, 2023.

\bibitem[Papini et~al.(2021)Papini, Tirinzoni, Pacchiano, Restelli, Lazaric,
  and Pirotta]{papini2021reinforcement}
Matteo Papini, Andrea Tirinzoni, Aldo Pacchiano, Marcello Restelli, Alessandro
  Lazaric, and Matteo Pirotta.
\newblock Reinforcement learning in linear {MDP}s: Constant regret and
  representation selection.
\newblock \emph{Advances in Neural Information Processing Systems},
  34:\penalty0 16371--16383, 2021.

\bibitem[Ren and Krogh(2002)]{RenKro02}
Zhiyuan Ren and B.H. Krogh.
\newblock State aggregation in {M}arkov decision processes.
\newblock In \emph{Proceedings of the 41st IEEE Conference on Decision and
  Control}, volume~4, pages 3819--3824, 2002.

\bibitem[Sen et~al.(2016)Sen, Shanmugam, Kocaoglu, Dimakis, and
  Shakkottai]{SSKDS-2016}
Rajat Sen, Karthikeyan Shanmugam, Murat Kocaoglu, Alexandros~G. Dimakis, and
  Sanjay Shakkottai.
\newblock Contextual bandits with latent confounders: An {NMF} approach, 2016.

\bibitem[Simchi-Levi and Xu(2022)]{simchi2022bypassing}
David Simchi-Levi and Yunzong Xu.
\newblock Bypassing the monster: A faster and simpler optimal algorithm for
  contextual bandits under realizability.
\newblock \emph{Mathematics of Operations Research}, 47\penalty0 (3):\penalty0
  1904--1931, 2022.

\bibitem[Slivkins(2019)]{slivkins2019introduction}
Aleksandrs Slivkins.
\newblock Introduction to multi-armed bandits.
\newblock \emph{Foundations and Trends{\textregistered} in Machine Learning},
  12\penalty0 (1-2):\penalty0 1--286, 2019.

\bibitem[Tewari and Murphy(2017)]{TM17}
A.~Tewari and S.~A. Murphy.
\newblock From ads to interventions: Contextual bandits in mobile health.
\newblock In \emph{Mobile Health}, pages 495--517. Springer, 2017.

\bibitem[Trinh et~al.(2020)Trinh, Kaufmann, Vernade, and
  Combes]{trinh2020solving}
Cindy Trinh, Emilie Kaufmann, Claire Vernade, and Richard Combes.
\newblock Solving {B}ernoulli rank-one bandits with unimodal {T}hompson
  sampling.
\newblock In \emph{Algorithmic Learning Theory}, pages 862--889. PMLR, 2020.

\bibitem[Uehara et~al.(2022)Uehara, Zhang, and Sun]{uehara2022representation}
Masatoshi Uehara, Xuezhou Zhang, and Wen Sun.
\newblock Representation learning for online and offline {RL} in low-rank
  {MDP}s.
\newblock In \emph{International Conference on Learning Representations}, 2022.

\bibitem[Zhang et~al.(2020)Zhang, Sodhani, Khetarpal, and
  Pineau]{zhang2020learning}
Amy Zhang, Shagun Sodhani, Khimya Khetarpal, and Joelle Pineau.
\newblock Learning robust state abstractions for hidden-parameter block {MDP}s.
\newblock In \emph{International Conference on Learning Representations}, 2020.

\bibitem[Zhang et~al.(2021)Zhang, He, Zhou, Zhang, and Gu]{zhang2021provably}
Weitong Zhang, Jiafan He, Dongruo Zhou, Amy Zhang, and Quanquan Gu.
\newblock Provably efficient representation learning in low-rank {M}arkov
  decision processes.
\newblock \emph{arXiv preprint arXiv:2106.11935}, 2021.

\bibitem[Zhang et~al.(2022)Zhang, Song, Uehara, Wang, Agarwal, and
  Sun]{zhang2022efficient}
Xuezhou Zhang, Yuda Song, Masatoshi Uehara, Mengdi Wang, Alekh Agarwal, and Wen
  Sun.
\newblock Efficient reinforcement learning in block {MDP}s: A model-free
  representation learning approach.
\newblock In \emph{International Conference on Machine Learning}, pages
  26517--26547. PMLR, 2022.

\end{thebibliography}

\appendix
\newpage
\section{More Related Works}\label{app: related works}
\subsection{From bilinear to low-rank bandits}
The low-rank structure has been shown useful in other bandit problems.
\citet{katariya2017stochastic} considered the case when the learner chooses both ``context'' and arm and the reward map $\rx$, when viewed as a $\nc \times \na$ matrix, has \emph{rank one}.
This is a special case of a linear bandit setting with a fixed action set where both the features underlying the actions
and the parameter vectors are ``shaped'' as matrices of matching dimensions,
and both matrices are rank-one. Similarly to our problem, the challenge is to
make the regret depend on $\nc + \na$ rather than on $\nc\times \na$ (but note that $\nc$ does not have the interpretation of the number of contexts here). This paper demonstrates that this is possible
when considering gap-dependent bounds.

\citet{trinh2020solving} improved the result of \citet{katariya2017stochastic}
and gave an asymptotically optimal regret bound.
\citet{kveton2017stochastic} generalized the setting of \citet{katariya2017stochastic} to the case when
 the reward $\rx$ (viewed as matrix) is of rank $\rk\le\min(\nc,\na)$
and is a ``hott topics matrix'', the learner in each round can choose $\rk$ row and column indices,
observes the entries of $\rx$ at the resulting submatrix of $\rx$ in noise and incurs the maximum reward  in this submatrix. For this problem they show an instance dependent bound that depends on $\nc, \na,\rk$ only through $(\nc+\na) \text{poly}(\rk)$.

\citet{jun2019bilinear} dropped the extra conditions on the mean reward matrix besides assuming that it has low rank.
\citet{jang2021improved} gave the first upper bound of the form
$\tilde O( (\nc + \na) \rk \sqrt{T})$, though with inefficient algorithms,
which was the first bound better than the naive bound $\tilde O(\nc\na \sqrt{T})$.%
\footnote{To get this naive bound,
following an argument of \citet{jang2021improved},
notice that the problem is an instance of
$d$-dimensional  linear bandits with $d=\nc\na$ with actions and parameters belonging to the unit sphere,
in which case the minimax regret, up to logarithmic factors, is of order $d\sqrt{T}$ (e.g.,
Theorem 24.2, \citealt{lattimore2021bandit}).
If the action set has one-hot matrices only, this bound improves to $\sqrt{dT}$, as discussed before.
}
Later, \citet{lu2021low} dropped the condition on the action matrices
and also extended
the results to the generalized linear setting (they assume that both the action matrices and the parameter matrix has a constant Frobenius norm bound).
The regret bound
of \citet{jun2019bilinear,lu2021low} takes the form $\tilde O( (\nc+\na)^{3/2} \sqrt{\rk T})$,
which is still worse than the earlier mentioned naive bound.
\citet{lattimore2021bandit} prove that (up to logarithmic factors)
when $\nc=\na$ and both the action and reward matrices are rank one and symmetric, the minimax regret is of order
$\nc \sqrt{T}$.

Recently, \citet{kang2022efficient} improved the previous state-of-the-art for the generalized linear setting. The new regret bound, which they prove ``under a mild'' extra assumption, takes the form
$\tilde O( (\nc + \na) \rk \sqrt{T})$ and
is  conjectured to match the order of the minimax regret.%
Noticing that the minimax lower bounds developed for the finite-armed stochastic bandits (e.g.,
Exercise~15.2 of \citealt{lattimore2020bandit}) is applicable to the rank-one setting (as the proofs use actions and rewards that are rank one, one-hot matrices), we see that the regret is at least $\Omega(\sqrt{\nc\na T})$ in these problems.
As $
\sqrt{\nc+\na} \approx
\sqrt{\max(\nc,\na)}\ll
\sqrt{\nc\na } \le \sqrt{\max(\nc,\na)^2} = \max(\nc,\na)\approx \nc + \na$,
this lower bound rules out upper bounds of the form $O(\sqrt{(\nc+\na)\text{poly}(r)T})$,
though it is compatible with the upper bound mentioned above.%
\footnote{Here, $f(\nc,\na) \approx g(\nc,\na)$ means that $f,g$ are within a constant factor of each other,
while $f(\nc,\na) \ll g(\nc,\na)$ means that $f(\nc,\na) = o(g(\nc,\na))$ as $\nc,\na\to\infty$.}

\citet{JP-2022} study a related problem where the learner chooses one action per context in each round and incurs a reward for each of the actions. They also propose an epsilon-greedy type algorithm and show that its regret scales as $T^{2/3}\polylog(S+K)$ provided some technical conditions hold, one of which is that the condition number of the reward $\rx$ when viewed as a matrix is constant. The main limitation of this work is that the technical conditions are restrictive and the problem is easier as in each round an observation is available for each context, whereas in our setting the contexts arrive at random from a distribution which may be very far from the uniform distribution, which, as we shall see, will require extra care.

\subsection{Bandit meta-learning}
Bandit meta-learning is concerned with learning a lower dimensional subspace across bandit tasks \citep{kveton-2020,KKZHMBS-2021,CLP-2020}, while we consider only a single task.

\subsection{Contextual bandit problems}
Contextual bandit problems, introduced by \citet{auer2002nonstochastic}, can be seen as a special case of the
\emph{prediction with expert advice problem}, which is an online problem.
In a prediction with expert advice problem, the learner is given access to the recommendations of $\NE$ experts in the form of distributions over the $\na$ actions.
The learner still needs to choose an action in each round with the goal of
competing with the total reward collected by the best expert in hindsight.
\citet{auer2002nonstochastic} consider the adversarial case when in each round, each action is assigned
a reward in an arbitrary way from a known, finite interval.
The authors propose the EXP4 algorithm that is shown to achieve
$O(\sqrt{\na T\log \NE})$ regret in $T$ rounds of interactions.
\citet{beygelzimer2011contextual} extended this result to control the regret with high probability, while \citet{dudik2011efficient} introduced a computationally efficient variant assuming a computationally efficient cost-sensitive classification oracle.

To reduce
an $\rk$-lumpable contextual bandit problem to prediction with expert advice,
we need to choose the experts. The obvious choice here is that an expert is a $[\nc]\to [\na]$ map
that is the composition of an $[\nc]\to[\rk]$ map followed by an $[\rk]\to[\na]$ map.
Denoting by $\NE$ the number of such experts, we see that $\log(\NE) = \nc\log( \rk)+\rk \log(\na)$,
and hence the regret of EXP4 is of order $\Omega(\sqrt{\nc \na  T})$,
which is not better than not using the lumpability structure.

Another line of work focuses on oracle-based contextual bandits \citep{foster2020beyond,simchi2022bypassing}.
In this setting, the learner has access to function class $\mathcal{F}$ and the optimal regret is $O(\sqrt{KT\log(|\mathcal{F}|)})$, where $|\mathcal{F}|$ is the size of the function class.
However, with the same argument, we can conclude $|\mathcal{F}|=\Omega(S^r)$ even under the lumpable assumption. 
Therefore, the results in this direction also give no direct implication to the problem we consider. 

\subsection{Stochastic linear bandits with changing action sets}
In this setting, in each round $t=1,2,\dots$ the learner first receives $\na$ $d$-dimensional vectors, $x_{t,1},\dots,x_{t,\na}$, each corresponding to an action.
Choosing action $\ac$ gives a reward whose mean (given the past) is $x_{t,\ac}^\top \theta$ for some unknown parameter vector $\theta\in \R^d$.
For our case, one can choose $d = \nc \na$: $\theta$ can be the ``flattening'' of $\rx$ and $x_{t,\ac}$ is a  unit vector so that $x_{t,\ac}^\top \theta = \rx(\cx_t,\ac)$.
Applying the SupLinRel algorithm from Section 4 of  \cite{Aue02} to this setting, we get the regret bound $O(\sqrt{ d T \log^3(\na T)} )$, which shows no improvement compared to ignoring lumpability.

\subsection{Matrix completion problems}
The offline version of our problem is closely related to matrix completion problems, where the goal is to reconstruct a matrix with missing values under the low-rank constraint \citep{arora2012learning, jain2013low, chen2020noisy}.
Based on these ideas, \citet{SSKDS-2016} propose an epsilon-greedy algorithm for the contextual low-rank bandit problem and show that its regret is of order $T^{2/3} (\nc\, \poly(\rk, \log \na))^{1/3}$. This result holds under an assumption that the reward matrix has a nonnegative decomposition $\rx=UV$ where entries of $U$ and $V$ are all nonnegative.
If $\rx$ is nonnegative valued itself, $\rk$-lumpability of the contexts implies that such a decomposition exist.
However,  the main result of this work needs additional conditions.
In particular, the context distribution needs to be near-uniform, and due to their ``restricted isometry property'' assumptions, the context groups need to be of approximately the same size. In particular,
if all context groups except one have a single member,  their bound degrades to the trivial bound mentioned earlier.

\subsection{Lumpable Markov decision processes}\label{app: MDP}
Lumpable MDPs in contemporary literature on machine learning
are referred to as \emph{block MDPs} \citep{du2019provably}.
Learning to act in a block MDPs has been the subject a numerous recent papers
\citep{dann2018oracle,du2019provably,misra2020kinematic,feng2020provably,zhang2020learning,zhang2022efficient}.
Furthermore, learning to act in a block MDP is a special case of the so-called low-rank setting, which was also heavily studied
\citep{jiang2017contextual,uehara2022representation,modi2021model,papini2021reinforcement,zhang2021provably,misra2020kinematic}, both in the online and in the PAC settings.
The learner in these problems is given a set of feature maps with the promise that at least one of the feature maps will allow a factored (low-rank) representation of the transition dynamics and the reward.

Despite the significant effort that went into studying this problem,
in our setting none of the existing results improve upon the naive bound that one can get for non-lumpable problems.
\citet{duan2019state,ni2021learning} study the offline setting, where the considerations are rather different.
Finally, \citet{kwon2021rl} consider the problem when over multiple episodes,
a learner interacts with a finite-horizon MDP, which is chosen at random from one of $\rk$ such unknown MDPs. The learner is given no additional information about the hidden identity of the MDP that it faces.
The challenge here is that even if the MDPs were known, the problem of acting optimally is nontrivial,
while in our case this is not a problem. As such, while superficially, the problems may look similar, they are quite different.
\section{Auxiliary Lemmas}
In this section, we show some probabilistic lemmas that will be useful to prove our results. 
The following lemma is a high-probability version of the classical ``coupon collector's problem'', which says that the expected number of coupons required to draw with replacement is $\Theta(K\log K)$ in order to get each of $K$ coupon at least once.
\begin{lemma}\label{lem:sample actions}
    Given a set $\mathcal{K}$ with $|\mathcal{K}|\le K$ and consider $M$ i.i.d samples drawn from $\unif{\mathcal{K}}$. Then with probability $1-\delta'$, every element in $\mathcal{K}$ appears at least once in these samples as long as $M\ge K\log(K/\delta')$.
\end{lemma}
\begin{proof}
    Fix an element $j\in\mathcal{K}$. The probability that $j$ appears in none of the $M$ samples can be bounded by
    \begin{align*}
        \left(1-\frac{1}{|\mathcal{K}|}\right)^M\le \left(1-\frac{1}{K}\right)^M\le\exp\left(-\frac{M}{K}\right)\le\frac{\delta'}{K}.
    \end{align*}
    A union bound on every $j$ in $\mathcal{K}$ finishes the proof. 
\end{proof}
\begin{lemma}[Concentration of Subgaussian random variables]\label{lem:subg}
    Let $X_1-\mu,\dots,X_M-\mu$ be a sequence of independent 1-subgaussian random variables and $\widehat{\mu}=\frac{1}{M}\sum^M_{m=1}X_m$. Then with probability $1-\delta'$, $\delta'<\frac{1}{2}$, we have
    \begin{align*}
        \left|\widehat{\mu}-\mu\right|\le 2\sqrt{\frac{\log(1/\delta')}{M}}
    \end{align*}
\end{lemma}
\begin{lemma}[Bernstein’s inequality]\label{lem:Bernstein}
    Let $X_1,\dots,X_M$ be a sequence of independent random variables with $X_m-\E[X_m]\le b$ almost surely for every $m$ and $v=\sum^M_{m=1}\mathrm{Var}[X_m]$.
    With probability $1-\delta'$, we have
    \begin{align*}
        \sum^M_{m=1}X_m \le \sum^M_{m=1} \E[X_m]+\sqrt{2v\log({1/\delta'})}+\frac{2b}{3}\log(1/\delta')
    \end{align*}
\end{lemma}
The following lemma is a direct consequence of Bernstein’s inequality, which states, in terms of the coupon collector's problem, that after drawing $M$ coupons, one can expect with a high probability a particular coupon appears at least $Mp/2$ times if its probability of appearing is $p$.
\begin{lemma}\label{lem: sample iid rvs}
    Let $X_1,\dots,X_M$ be i.i.d Bernoulli random variables so that $\E[X_m]=p$ for all $m$.
    With probability $1-\delta'$, we have $\sum_{m=1}^MX_m\ge M p/2$ as long as $Mp\ge 16\log(1/\delta')$.
\end{lemma}
\begin{proof}
    Let $Y_m=1-X_m$ We have $\E[Y_m]= 1-p$. By \pref{lem:Bernstein}, with probability $1-\delta'$, we have
    \begin{align*}
    M-\sum_{m=1}^MX_m\le M\cdot\left(1-p\right)+\sqrt{2Mp(1-p)\log({1/\delta'})}+\log(1/\delta').
    \end{align*}
    Rearranging terms and using $Mp\ge 16\log(1/\delta')$ gives
    \begin{align*}
    \sum_{m=1}^MX_m&\ge Mp-\sqrt{2Mp(1-p)\log({1/\delta'})}-\log(1/\delta')\\
    &\ge Mp-\frac{Mp}{\sqrt{8}}-\frac{Mp}{16}\\
    &\ge\frac{Mp}{2}
    \end{align*}
\end{proof}

\section{Proofs for Section \ref{sec:pac}}\label{app:pac-proof}
We first provide an outline of the proof in \cref{sec:pac-proof-overview} and show the complete proofs of the lemma in the corresponding subsections.
\subsection{Proof of Theorem \ref{thm:pac}}\label{sec:pac-proof-overview}

To begin with, we assign adaptive target optimality (accuracy) to each block according to their block size:
\[
\epsilon_b = \max\left\{1,\frac{1}{\sqrt{r\omega(b)}}\right\}\epsilon, \qquad \text{for } b\in[r].
\]

The following lemma states that as long as $\cA$ contains an $\Otilde{\epsilon_b}$-optimal action for each block $b\in[r]$, then the output policy is
$\Otilde{\epsilon}$-optimal with high probability. 
\begin{lemma}\label{lem:pac-1}
For a positive constant $C$, suppose for any $b\in[r]$ with $\omega(b)\ge \epsilon/r$, $\mathcal{W}$ contains a $C\epsilon_b$-optimal action, %
then $\piout$ is $2(C+1)\epsilon$-optimal for the original context-lumpable bandit problem.
\end{lemma}

As a result, if we can prove that the precondition  of \cref{lem:pac-1} holds with high probability, then the correctness of Algorithm \ref{alg:pac-unif-context} follows immediately. 
To do so, we first argue that for every block $b\in[r]$, we have sufficiently explored the entire action set at accuracy level $n=\lceil\log(1/\epsilon_b^2)\rceil$ in the data collection step, as stated in the next lemma.

\begin{lemma}\label{lem:pac-2}
   With probability at least $1-2\delta$, for any $b\in[r]$ with  $\omega(b)\ge \epsilon/r$, we have that: for any $j\in[K]$, there exists $(i,j)\in\mathcal{D}_n$ satisfying $g(i)=b$ where $n=\lceil\log(1/\epsilon_b^2)\rceil$.
\end{lemma}

Intuitively, \cref{lem:pac-2} says that we have tested all actions on each block $b\in[r]$ at the corresponding accuracy level $n=\lceil\log(1/\epsilon_b^2)\rceil$. Based on \cref{lem:pac-2}, the following lemma states that in the second step we are able to filter out an $\Otilde{\epsilon_b}$-optimal action for block $b$ by using  the information contained in $\cD_n$.

\begin{lemma}\label{lem:pac-4}
   With probability at least $1-4\delta$, for any $b\in[r]$ with  $\omega(b)\ge \epsilon/r$, a $10\epsilon_b\sqrt{\log(rSK/\delta)}$-optimal action of block $b$ is added into $\mathcal{W}$ in the process of shrinking $\cD_n$ where $n=\lceil\log(1/\epsilon_b^2)\rceil$.
\end{lemma}

Now we have proved the correctness of Algorithm \ref{alg:pac-unif-context}. However, to obtain the desired sample complexity, it remains to show that the size of the optimal action candidate set $\cA$ is relatively small, which we prove next.

\begin{lemma}\label{lem:pac-3}
    With probability at least $1-4\delta$, $|\mathcal{W}|\le Nr$. 
\end{lemma}

Now we are ready to prove Theorem \ref{thm:pac} by 
combining all the above lemmas.

\begin{proof}[Proof of Theorem \ref{thm:pac}]
    The step of data collection uses 
    \[
    8\left(1+\log\frac{1}{\epsilon^2}\right)\cdot\left(2+\log\frac{1}{\epsilon^2}\right)\cdot \left(\log\frac{rSK}{\delta}\right)\cdot \frac{r(S+K)}{\epsilon^2}
    \]
    samples. By \cref{lem:pac-3}, the step of screening optimal action candidates uses
    \[
    2(16^2)\left(1+\log\frac{1}{\epsilon^2}\right)\cdot\left(\log\frac{rSK}{\delta}\right) \cdot \frac{rS}{\epsilon^2}
    \]
    samples and solving the simplified context-lumpable bandit problem uses 
    \[
    4\left(1+\log\frac{1}{\epsilon^2}\right)\cdot \left( \log\frac{SK}{\delta} \right) \cdot \frac{Sr}{\epsilon^2}
    \]
    samples. Overall, the number of samples is bounded by
    \[
    524\left(\log\frac{rSK}{\delta}\right) \cdot \left(1+2\log\frac{1}{\epsilon}\right)^2\cdot \frac{r(S+K)}{\epsilon^2}
    \]
    By \cref{lem:pac-1} and \cref{lem:pac-4}, $\piout$ is $2\left(10\sqrt{\log(rSK/\delta)}+1\right)\epsilon$-optimal.
\end{proof}

\subsection{Proof of \texorpdfstring{\cref{lem:pac-1}}{}}

\begin{proof}
We control the suboptimality of $\piout$ in the following way:
\begin{align*}
    & \sum_{i\in[S]}  \nu(i)  \left( \max_{j\in[K]} A(i,j) - A(i,\piout(i))\right) \\
   \le & \sum_{b\in[r]:~\omega(b)\ge \epsilon/r} \sum_{i\in\cls(b)} \nu(i)  \left( \max_{j\in\mathcal{W}} A(i,j) - A(i,\piout(i)) +C\epsilon_{b}\right)+\sum_{b\in[r]:~\omega(b)< \epsilon/r} \omega(b) \\
   \le & \epsilon+  C\sum_{b\in[r]} \omega(b) \epsilon_{b}+\epsilon \le 2(C+1)\epsilon,
\end{align*}
where the first inequality uses the precondition of the lemma, the second one uses the $\epsilon$-optimality of $\piout$ in the simplified context-lumpable bandit, and the final one follows from Cauchy-Schwarz inequality. 
\end{proof}

\subsection{Proof of \texorpdfstring{\cref{lem:pac-2}}{}}

\begin{proof}
    Fix a block $b\in[r]$.
    For each accuracy level, recall in the step of data collection, we sample $\nes=r(\nc+\na)\lg/\epsilon^2$ contexts i.i.d. from $\nu$.
    By~\cref{lem: sample iid rvs}, with probability $1-\frac{\delta}{r}$, at least $\nes\omega(\cl)/2$ of them are from block $b$ as
    $$
    \nes\omega(\cl)\ge \frac{\nes\epsilon
    }{r}\ge {\lg}\ge 16\log(r/\delta).
    $$
    Further, recall that we define accuracy level $n=\lceil\log(1/\epsilon_b^2)\rceil$.
    We first show that this $n$ is well defined, that is, $n\ge1$. Indeed, we have
    \begin{align*}
       \frac{1}{\epsilon_b^2} = \frac{1}{\epsilon^2/(r\db(b))}\ge \frac{\epsilon}{\epsilon^2}=\frac{1}{\epsilon}\ge 2
    \end{align*}
    as long as $\epsilon\le \frac{1}{2}$.
    Since we add a context-action pair into $\mathcal{D}_n$ once we have collected $$2^n=2^{\lceil\log(1/\epsilon_b^2)\rceil}< 2^{\log(1/\epsilon_b^2)+1}=2/\epsilon_b^2$$
    samples for estimating its reward. Note that in the end there are at most $|\cls(b)|(2^n-1)$ samples from block $b$ unused.
    Thus, with probability $1-\frac{\delta}{r}$, 
    there are at least ${\nes\omega(b)/2-|\cls(b)|(2^n-1)}$ samples used. 
    Therefore, the number of context-action pairs,  where the contexts are from block $b$, that are added into $\mathcal{D}_n$ is at least
    \begin{align*}
        \frac{\nes\omega(b)/2-|\cls(b)|(2^n-1)}{2^n} &\ge \frac{\nes\omega(b)/2}{2/\epsilon_b^2}-S\tag{$2^n\le 2/\epsilon_b^2$ and $|\cls(b)|\le S$}\\
        &\ge\frac{\nes\epsilon^2/(2r\epsilon_b^2)}{2/\epsilon_b^2}-S \tag{by definition of $\epsilon_b$}\\
        &=\ge 16(\nc+\na)\log(rSK/\delta)-S\tag{the value of $\nes$}\\
        &\ge K\log(rSK/\delta).
    \end{align*}
    Conditioned on this event, with probability $1-\frac{\delta}{r}$, for any $j\in[K]$, there exists $(i,j)\in\mathcal{D}_n$ satisfying $g(i)=b$ by \cref{lem:sample actions}.
    Therefore, the lemma holds for block $b$ with probability at least $1-\frac{2\delta}{r}$.
    We complete the proof by a union bound on all blocks. 
\end{proof}

\subsection{Proof of \texorpdfstring{\cref{lem:pac-4}}{}}

\begin{proof}
Denote by $i$ the first context  being removed from $\mathcal{D}_n$, which satisfies $g(i)=b$. 
By the rule of shrinking $\mathcal{D}_n$, we have 
$$
|\widetilde A_n(i,j^\star)-  \widehat A_n(i^\star,j^\star)|\le 4\sqrt{\frac{\log(rSK/\delta)}{2^n}}.
$$
By \cref{lem:subg} and a union bound on all pairs in $\mathcal{D}_n$, with probability $1-2{\delta}/r$, we have 
\begin{align}
    |\widetilde A_n(i'',j'')-A(i'',j'')|\le \frac{1}{2}\sqrt{\frac{\log(rSK/\delta)}{2^n}} \quad \text{and}\quad|\widehat A(i'',j'')-A_n(i'',j'')|\le \frac{1}{2}\sqrt{\frac{\log(rSK/\delta)}{2^n}}\label{eq:pac-good-event}
\end{align}
for every $(i'',j'')\in \mathcal{D}_n$.
Therefore, we have
\begin{equation}\label{eq:pac-1}
    | A(i,j^\star)-  A(i^\star,j^\star)|\le 5\sqrt{\frac{\log(rSK/\delta)}{2^n}},
\end{equation}

By \cref{lem:pac-2}, we know that for any $j\in[K]$, there exists $(i',j)\in\mathcal{D}_n$ satisfying $(g(i'),j)=(b,j)$,  before we remove context $i$.
Therefore, by the definition of $(i^\star,j^\star)$, for any $j\in[K]$, there exists $(i',j)\in\mathcal{D}_n$ satisfying $g(i')=b$ and 
\begin{equation*}
  \widehat A_n(i^\star,j^\star) \ge \widehat A_n(i',j),  
\end{equation*}
which, again by \cref{eq:pac-good-event} with probability $1-2{\delta}/r$
\begin{equation}\label{eq:pac-2}
  A(i^\star,j^\star) \ge  \max_{j} \mu(b,j)-\sqrt{\frac{\log(rSK/\delta)}{2^n}}.
\end{equation}
By combining \cref{eq:pac-1} and \cref{eq:pac-2}, we conclude that $j^\star$ is a $10\sqrt{\log(rSK/\delta)}\epsilon_b$-optimal action for block $b$ because
$$
5\sqrt{\frac{\log(rSK/\delta)}{2^n}}\ge 10\sqrt{\frac{\log(rSK/\delta)}{1/\epsilon_b^2}}= 10\sqrt{\log(rSK/\delta)}\epsilon_b
$$
This completes the proof by a union bound on all blocks.
\end{proof}

\subsection{Proof of \texorpdfstring{\cref{lem:pac-3}}{}}

\begin{proof}
    It suffices to show that for each accuracy level $n$, we will add at most $r$ actions into $\mathcal{W}$ before shrinking $\mathcal{D}_n$ to $\emptyset$. Below we prove a stronger argument: each time we shrink $\mathcal{D}_n$, we will remove all the contexts from at least one block from $\mathcal{D}_n$. %
    
    Denote by $(i,j)$ an arbitrary context-action pair from $\mathcal{D}_n$ such that $g(i)=g(i^\star)$, then by \cref{eq:pac-good-event}, we have%
    \begin{align*}
    \left|\widetilde A_n(i,j^\star)-  \widehat A_n(i^\star,j^\star)\right|
    \le  \left|\widetilde A_n(i,j^\star)-  A(i,j^\star)\right|+ \left| A(i^\star,j^\star)-  \widehat A_n(i^\star,j^\star)\right|
    \le 4\sqrt{\frac{\log(rSK/\delta)}{2^n}},
    \end{align*}
    which implies all context-action pairs with context from block $g(i^\star)$ will be removed from $\mathcal{D}_n$. 
\end{proof}

\subsection{Proof of Theorem \ref{thm:pac-general} }

\begin{proof}
    By \cref{lem: sample iid rvs} and a union bound over $[S]$, we have that with probability $1-\delta$ for all $i\in[S]$:
    $$
    \frac12\left(\nu(i) - \frac{\log(S/\delta)}{J}\right) \le \hat\nu(i) \le 2\left(\nu(i) + \frac{\log(S/\delta)}{J}\right), 
    $$
which implies that $\nu(i)\ge \epsilon/(4S)$  %
for $\hat\nu(i)\ge \epsilon/S$ and $\nu(i)\le 3\epsilon/S$ %
for $\hat\nu(i)<  \epsilon/S$.  
By definition, for any $l\in[L-1]$, $\mathcal{X}_l$ consists of contexts such that $\hat\nu(i)>2^{-l-1}\ge \epsilon/S$. 
As a result, for any $l\in[L-1]$, by Lemma~\ref{lem: sample iid rvs}, we have that at least $N \nu(\mathcal{X}_l)/2$ %
out of $N$ samples corresponds to $\mathcal{X}_l$ and thus can be used in learning $\pi_l$ by executing Algorithm \ref{alg:pac-unif-context},  where $\nu(\mathcal{X}_l):=\sum_{i\in\mathcal{X}_l} \nu(i) \ge \epsilon/(4S)$.
Moreover, notice that inside each $\cX_l$, the conditional context distribution is almost uniform, so we can invoke 
 Theorem \ref{thm:pac} to obtain that $\pi_l$ is 
\[
\frac{C\epsilon}{\sqrt{\nu(\mathcal{X}_l)}}\text{-optimal}
\]
over context $\mathcal{X}_l$, where $C=2\sqrt{2}(10\sqrt{\log(rSK/\delta)}+1)$. 
This means that, 
\[
\sum_{i\in\cX_l} \frac{\nu(i)}{\nu(\cX_l)}\left(\max_j A(i,j)- A(i,\pi_l(i))\right) = \frac{C\epsilon}{\sqrt{\nu(\mathcal{X}_l)}}.
\]
As a result, the total suboptimality of $\piout$ can be upper bounded as following
\begin{align*}
    & \sum_{i\in[S]}  \nu(i)  \left( \max_{j\in[K]} A(i,j) - A(i,\piout(i))\right) \\
   \le & \sum_{l\in[L-1]} \sum_{i\in\cX_l} \nu(i)  \left( \max_{j} A(i,j) - A(i,\pi_l(i))\right)+\sum_{i\in\cX_L} \nu(i) \\
   \\
   \le & C\sum_{l\in[L-1]}\nu(\mathcal{X}_l)\times \frac{\epsilon}{\sqrt{\nu(\mathcal{X}_l)}} +\epsilon= C\sqrt{L}\epsilon,
\end{align*}
where the final equality uses Cauchy-Schwartz inequality and $L\le \log(S/\epsilon)+1$.
\end{proof}

\subsection{Proof of Theorem \ref{thm:pac lower-bound}}

\begin{proof}
Since $r\le S$, we have
$$
\Omega(\min\{r(S+K),SK\})
=\begin{cases}
\Omega(SK),\quad &r\ge K,\\
\Omega(rS),  &r<K ~\& ~S\ge K,\\
\Omega(rK),  &r<K ~\& ~S< K.
\end{cases}
$$

\paragraph{Case (1)} We construct the following hard instance:
\begin{itemize}
\item The reward  after pulling action $j\in[K]$ in block $b\in[r]$  is sampled from distribution ${\rm Bernoulli}(1/2+\epsilon\mathbbm{1}(b=j))$.  
    \item The context distribution is uniform over $[S]$. For each $i\in[S]$, we sample $g(i)$ uniformly at random from $[K]$.
\end{itemize}
The above instance is the standard one commonly used in proving lower bounds for contextual bandit with $S$ contexts and $K$ arms \citep[e.g.,][]{lattimore2020bandit}. And learning an $\epsilon$-optimal policy for the above hard instance requires at least $\Omega(SK/\epsilon^2)$ samples.

\paragraph{Case (2)} We only need to slightly modify the above hard instance:
\begin{itemize}
\item The reward  after pulling action $j\in[K]$ in block $b\in[r]$  is sampled from distribution ${\rm Bernoulli}(1/2+\epsilon\mathbbm{1}(b=j))$.  
    \item The context distribution is uniform over $[S]$. For each $i\in[S]$, we sample $g(i)$ uniformly at random from $[r]$.
\end{itemize}
The above modified problem is equivalent to the original one with $S$ contexts but $r$ arms. As a result,  a lower bound of form  $\Omega(Sr/\epsilon^2)$ holds.

\paragraph{Case (3)} Similarly, we slightly modify the first hard instance:
\begin{itemize}
\item For each block $b\in[r]$, sample $j^\star_b$ uniformly at random from $[K]$.  The reward  after pulling action $j\in[K]$ in block $b\in[r]$  is sampled from distribution ${\rm Bernoulli}(1/2+\epsilon\mathbbm{1}(j=j_b^\star))$.  
    \item The context distribution is uniform over $[r]$ and $g(i)= \min\{i,r\}$.  
\end{itemize}
The above modified problem is equivalent to the original one with $K$ arms but $r$ contexts. As a result,  a lower bound of form  $\Omega(Kr/\epsilon^2)$ holds.

\end{proof}

\section{Proof of \texorpdfstring{\cref{thm: regret multiple}}{}}\label{app: regret multiple blocks}%

In this section, we first show key lemmas to prove \cref{thm: regret multiple}.
The proofs of the lemmas are deferred to \cref{app: regret multiple blocks proofs}.

In the following discussion, we consider a single phase (i.e. fix an error $\epsilon_h$).
Similar to the analysis of other phased elimination algorithms, we have to show that in a phase specified by error level $\epsilon_h$, with high probability, (i) the optimal arm is not eliminated and (ii) all $\omega(\epst_h)$-suboptimal arms are eliminated, that is, all arms in $\good_h(\cs)$ for all $\cs$ are  ${O}(\epst_h)$-optimal.

The following lemma is the counterpart of \cref{lem:pac-2}, which ensures that every arm is played in every block. 
\begin{lemma}\label{lem: every arm is played}
With probability at least $1-2\delta_h$, for any cluster $\cs\in\partition_h$ and any block $b$ so that $\cls(b)\subseteq \cs$, we have that: for any $\ac\in\good_h(\cs)$ there exists $(i,j)\in\mathcal{D}_h$ satisfying $g(i)=b$.
\end{lemma}
The above lemma ensures that every block has a least one context $\cx$ assigned to include $(i,\ac)$ in $\cD_h$.
Thus, every arm is explored for every block.

Next, we define an event $\goodevent_h$ under which the estimates $\rxh_h$ are good:
\begin{align*}
\goodevent_h &= \left\{ |\rxh_h(i,j) - \rx(i,j)|\le{\frac{\epst_h}{4}}{},~\text{$\forall(i,j)\in \cD_h$ } \right\} \;.
\end{align*}
The level of precision $\frac{\epst_h}{4}$ is more accurate than the elimination step and will be helpful in the analysis.
The following lemma states that $\goodevent_h$ is a high probability event. 

\begin{lemma}\label{lem: good event holds}
Event $\goodevent_h$ holds with probability $1-\delta_h$.    
\end{lemma}

Next, let $\ac_\cl=\argmax_{\ac\in[\na]}\mu(\cl,\ac)$ be the optimal arm in block $\cl$. %
The next lemma says that the optimal arm $\ac_\cl$ is not eliminated during the execution of the algorithm. 
\begin{lemma}\label{lem: all jc are not eliminated}
Assume action $\ac_\cl\in\good_h(\cs)$ for every block $\cl\in[\rk]$, and its corresponding partition $\partition_h \ni\cs\supseteq\cls(\cl)$. Then for every block $\cl\in[\rk]$, $\ac_\cl$ is not eliminated from $\good_{h+1}(\cs)$ with probability at least $1-3\delta_h$.
\end{lemma}
The high-level idea of the proof is that the error of the estimated mean is smaller than $\epst_h$, so $\rxh_t(i,\ac_\cl)$ will not be much worse than other arms, given that its true mean is largest. The next lemma shows that arms in $\good_{h+1}(\cs)$ are all $O(\epsilon_h)$-optimal. %
Formally, we say an arm $\ac$ in a block $\cl$ is $\epsilon$-optimal if $\max_{\ac'}\mu(\cl,\ac')-\mu(\cl,\ac)\le\epsilon$. Similarly, we say an arm $\ac$ in a block $\cl$ is $\epsilon$-suboptimal if $\max_{\ac'}\mu(\cl,\ac')-\mu(\cl,\ac)\ge\epsilon$.
\begin{lemma}\label{lem: all bads are eliminated - multiple blocks}
For any block $\cl\in[\rk]$ and its corresponding cluster $\partition_h\ni\cs\supseteq\cls(\cl)$, all $3\epst_h$-suboptimal arms in block $\cl$ are eliminated in $\good_{h+1}(\cs)$ with probability at least $1-3\delta_h$. Consequently, only $6\epst_h$-optimal arms in block $\cl$ are played in phase $h+1$ for every context in block $b$. 
\end{lemma}

Finally, we need to argue that \cref{alg: the cluster algorithm - multiple} is not called too many times.
To show this, we first provide a general guarantee of \cref{alg: the cluster algorithm - multiple}. 
\begin{lemma}\label{lem: clustering}
    Suppose we have context $i_u\in\cls(b_u)\subseteq\cs$ in block $b_u$ and context $i_l\in\cls(b_l)\subseteq\cs$ in block $b_l$ so that
    $$
    \rx(i_u,j)-\rx(i_l,j)=\mu(b_u,j)-\mu(b_l,j)>\frac{3r}{2}{\sqrt{\lg}\cdot\epsilon'}.
    $$
    Then \cref{alg: the cluster algorithm - multiple} separates $\cls(b_u)$ and $\cls(b_l)$ perfectly with probabilty at least $1-2\delta'$.
    In other words, there exist indices $c_u$ and $c_l$, $c_u\ne c_l$, such that $\cls(b_u)\subseteq\partition_{c_u}$ and $\cls(b_l)\subseteq\partition_{c_l}$.    
\end{lemma}
Consequently, with a high probability, \cref{alg: multiple blocks} makes progress in terms of clustering contexts every time calling \cref{alg: the cluster algorithm - multiple} and the number of calls is bounded by $r$, as shown in the following lemma. 
\begin{lemma}\label{lem: clustering times- multiple}
When \cref{alg: the cluster algorithm - multiple} is called by \cref{alg: multiple blocks}, it separates at least two blocks and never separates contexts in the same block with probability at least $1-5\delta_h$.
    Consequently, \cref{alg: the cluster algorithm - multiple} is called at most $\rk$ times.
\end{lemma}
We are now ready to bound regret.

\begin{proof}[Proof of \cref{thm: regret multiple}]
    Since there are at most $O(\log_T)$ phases, it suffices to bound regret at a single phase $h$. 
    Conditioned on all the high probability good events in the previous lemmas, we first bound the total number of timesteps spent in phase $h$.
    In the data collection stage, we use at most $\nes_h=\tfrac{r(S+K)\lg_h}{\epsilon_h^2}$ samples;
    as for the clustering stage, by \cref{lem: clustering times- multiple}, we know the total length of executing \cref{alg: the cluster algorithm - multiple} is at most
    $$
    \rk L' =r{{\nc\lg}/{\epsilon'^2}}= {{16\nc\rk^3\lg}/{\epsilon_h^2}}=\Otilde{\frac{\rk^3\nc}{\epsilon_h^2}}.
    $$
    Thus, the length of phase $h$ is the minimum of $T$ and $\Otilde{\frac{\rk^3(\nc+\na)}{\epsilon_h^2}}$.
    Therefore, by \cref{lem: all bads are eliminated - multiple blocks}, the regret is at most (recall that $\epst_h=\sqrt{\log_h}\cdot\epsilon_h=\otilde{\epsilon_h}$)
    \begin{align*}\label{eq: sqrt regret}
    6\epst_h \cdot \min\left\{T,\Otilde{\frac{\rk^3(\nc+\na)}{\epsilon_h^2}}\right\}&= \min\left\{{6T\epst_h},\Otilde{\frac{\rk^3(\nc+\na)}{\epsilon_h}}\right\}\\
    &= \min\left\{\Otilde{T\epsilon_h},\Otilde{\frac{\rk^3(\nc+\na)}{\epsilon_h}}\right\}\\
    &=\Otilde{\sqrt{\rk^3(\nc+\na)T}} \;.
    \end{align*}
    On the other hand, the bad events happen with probability $O({\delta_h})=O({\epsilon_h^2/(r^3SK)})$.
    In this case the regret contributes at most $\otilde{1}$.
\end{proof}

\section{Missing Proofs in \texorpdfstring{\cref{app: regret multiple blocks}}{}} \label{app: regret multiple blocks proofs} 

\subsection{Proof of \texorpdfstring{\cref{lem: every arm is played}}{}}

\begin{proof}
Note that under the uniform block assumption, we have $\db(b)=\frac{1}{r}\ge\frac{\epsilon_h}{r}$.
Thus, the proof follows directly from the proof of \cref{lem:pac-2} when $\epsilon_b=\epsilon_h$, $\delta=\epsilon_h^2$, and $n=n_h$.
The only difference is when applying \cref{lem:sample actions}, \cref{lem:pac-2} uses $\mathcal{K}=[K]$ but we need $\mathcal{K}=\good_h(\cs)$ here. 
\end{proof}

\subsection{Proof of \texorpdfstring{\cref{lem: good event holds}}{}}

\begin{proof}
Fix an $(i,j)$ pair. Applying \cref{lem:subg}, we have with probability $1-\frac{\delta_h}{SK}$,
\begin{align*}
    \left|\rxh_h(i,j)-A(i,j)\right|\le 2\sqrt{\frac{\log(SK/\delta_h)}{2^{(n_h)}}}\le \frac{\epst_h}{4}.
\end{align*}
We complete the proof by applying a union bound on all $(i,j)$ pairs in $\cD_h$.
\end{proof}

\subsection{Proof of \texorpdfstring{\cref{lem: all jc are not eliminated}}{}}

\begin{proof}
   We prove by contradiction.
   Assume that $\ac_\cl$ is removed from $\good_h(\cs)$.
   Let 
   $$\ac'=\displaystyle\argmax_{\ac\in\good_h(\cs)}\displaystyle\muu_h(\cs,\ac)$$
   be the action achieving the highest empirical mean. 
   By \cref{lem: every arm is played} there exists a context $\cx'\in\cls(\cl)$ so that $(i',j')\in\cD_h$.
   Moreover, the assumption that $\ac_\cl$ is eliminated implies that the condition at \cref{line: detect large gaps - multiple} of \cref{alg: multiple blocks} does not hold for the current partition $\partition_h$, which further implies that there exists context $\overline{i}\in\cs$ so that 
   \[
   \rxh_h(\overline{i},\ac') -\rxh_h(\cx',\ac') = \muu_h(\cs,\ac')-\rxh_h(\cx',\ac')< \epst_h \;.
   \]
   On the other hand, by the assumption that $\ac_\cl$ is eliminated, again by \cref{lem: every arm is played} we have that there exists a context $\cx''\in\cls(\cl)$, $(\cx'',\ac_\cl)\in\cD_h$ such that
   $$
    \muu_h(\cs,\ac')-\rxh_h(\cx'',\ac_\cl) \ge \muu_h(\cs,\ac')-\muu_h(\cs,\ac_\cl) > 2\epst_h \;.
   $$
   Combining two inequalities, we have
   $$
   \rxh_h(\cx'',\ac_\cl)+\epst_h < \muu_h(\cs,\ac')-\epst_h<\rxh_h(\cx',\ac').
   $$
   This means that $\mu(\cl,\ac_\cl)+\frac{\epst_h}{2}<\mu(\cl,\ac')$ under $\goodevent_h$, which contradicts the optimality of $\ac_\cl$.
   Therefore, we conclude that $\ac_\cl$ is not eliminated. 
\end{proof}

\subsection{Proof of \texorpdfstring{\cref{lem: all bads are eliminated - multiple blocks}}{}}

\begin{proof}
   By \cref{lem: all jc are not eliminated}, $\ac_\cl$ is not eliminated for any block $\cl\in[\rk]$ with probability at least $1-3\delta_h$. 
   Thus, for any block $\cl\in[\rk]$ and arm $\ac\in\good_h(\cs)$ so that $\mu(\cl,\ac_\cl)-\mu(\cl,\ac)>3\epst_h$, we have
   \begin{align*}
       \max_{\ac'}\muu_h(\cs,\ac')-\muu_h(\cs,\ac)\ge\muu_h(\cs,\ac_\cl)-\muu_h(\cs,\ac)\ge \mu(\cl,\ac_\cl)-\mu(\cl,\ac)-2\cdot\frac{\epst_h}{4} > 2\epst_h \;.
   \end{align*}
   Therefore, $\ac$ is eliminated at \cref{line: perform arm elimination - multiple}. 
\end{proof}

\subsection{Proof of \texorpdfstring{\cref{lem: clustering}}{}}
\begin{proof}
    With probability $1-\frac{\delta_h}{S}$, context $i$ receives at least $2/\epsilon'^2\ge 2^{n'}$ samples by \cref{lem: sample iid rvs}.
    Thus, $\rxh(i,j)$ is well defined for every $i\in \cs$ with probability $1-\delta_h$.
    Moreover, by \cref{lem:subg} and a union bound, we have for every context $i$, with probability at least $1-\delta_h$,
    \begin{align}
        \left|\rx(i,j)-\rxh(i,j)\right|\le \frac{\sqrt{\lg}\cdot\epsilon'}{4}\label{eq:clustering concentration}
    \end{align}
    Clearly, we have $\rxh(i_u)\ge\rxh(i_l)$ under \cref{eq:clustering concentration}.
    Consequently, to simplify the notation, we do the following modification on labels of contexts and blocks. 
    First, we restrict the game to $\cs$, where there are $S'$ contexts and $r'$ blocks;
    also, we relabel contexts so that $i_u=1$, $i_l=S'$, and $$\rxh(i_u)=\rxh(1)\ge\rxh(2)\ge\cdots\ge\rxh(S'-1)\ge\rxh(S')=\rxh(i_l).$$
    Finally, given a context $i\in[S']$, we define
    $$
    \muu(b)=\max_{i'\in[S'], g(i')=b}\rxh(i',j) \quad\text{and}\quad \mul(b)=\min_{i'\in[S'], g(i')=b}\rxh(i',j)
    $$
    for its block $b=g(i)$ and we relabel blocks so that $b_u=1\le g(i)\le r'=b_l$ and
    $$
    \muu(b_u)= \muu(1)\ge \muu(b_u-1)\ge\cdots \ge\muu(b_l+1)\ge \muu(r')=\muu(b_l).
    $$
    It is not hard to see that this modification is without loss of generality. 
    We show next that there exists a context $i$, $i_u<i\le i_l$, such that \cref{line: clustering large gap} of \cref{alg: the cluster algorithm - multiple} holds, that is, 
    \begin{align}
           {\rxh({i-1,j})-\rxh(i,j)\ge\sqrt{\lg}\cdot \epsilon'}. \label{eq:exist k large}
    \end{align}
    We prove this by contradiction. Assume \cref{eq:exist k large} doesn't hold for any $k$. 
   
    Then we have
    \begin{align*}
        \rx(i_u,j)-\rx(i_l,j)&\le |\rx(i_u,j)-\muu(b_u)|+|\rx(i_l,j)-\muu(b_l)|+\muu(b_u)-\muu(b_l)\\
        &= \frac{\sqrt{\lg}\cdot\epsilon}{4}+\frac{\sqrt{\lg}\cdot\epsilon}{4}+\sum_{b=b_u+1}^{b_l}\muu(i_{b-1})-\muu(i_{b})\\
        &< \frac{\sqrt{\lg}\cdot\epsilon}{2}+\sum_{b=b_u+1}^{b_l}\muu(i_{b-1})-\mul(i_{b})+\sqrt{\lg}\cdot\epsilon'\\
        &\le \frac{\sqrt{\lg}\cdot\epsilon}{2}+\frac{3}{2}\cdot (r-1)\sqrt{\lg}\cdot\epsilon'\\
        &< \frac{3}{2}\cdot r\sqrt{\lg}\cdot\epsilon',
    \end{align*}
    which contradicts the condition that $\rx(i_u,j)-\rx(i_l,j)\ge \frac{3r}{2}\sqrt{\lg}\cdot\epsilon$. Therefore, we conclude that \cref{eq:exist k large} holds for some $k$.
\end{proof}

\subsection{Proof of \texorpdfstring{\cref{lem: clustering times- multiple}}{}}

\begin{proof}
   The proof follows directly from \cref{lem: clustering} and the fact that \cref{alg: multiple blocks} use $\epsilon'=\frac{\epsilon_h}{4r}$ and thus
   \begin{align*}
        \rx(\overline{i},j)-\rx(\underline{i},j)\ge \frac{\epst_h}{2}\ge \frac{\sqrt{\lg_h}\cdot\epsilon_h}{2}>\frac{3}{2}\cdot r\sqrt{\lg}\cdot\epsilon',
   \end{align*}
   which satisfies the condition of \cref{lem: clustering}.
\end{proof}

\section{Non-uniform Context Distribution}\label{app: context non-uniform}
We show a reduction to problems with approximately uniform context distributions. The cost of this reduction is an extra $\otilde{\sqrt{\nc T}}$ additive regret and an extra $O(\log(\nc T))$ multiplicative factor in regret. The idea is to learn the context distribution in the first $\otilde{\sqrt{\nc T}}$ timesteps.
With high probability, we can estimate $\dc(\cx)$ for any context $\cx$ with a constant multiplicative error as long as $\dc(\cx)=\widetilde \Omega(1/\sqrt{\nc T})$.%
Then we split the contexts into several buckets so that contexts within the same bucket have the same probability up to a constant factor.
For contexts $\cx$ with $\dc(\cx)=o(1/\sqrt{\nc T})$, we can not estimate the probability properly but we can simply ignore such contexts and suffer regret at most $O(T\cdot \nc \cdot 1/\sqrt{\nc T})=O(\sqrt{\nc T})$.
We then run the algorithm that handles uniform context distribution for each bucket separately.
Since there are $O(\log(\nc T))$ buckets, the overall regret is $O(\log(\nc T))$ times maximum regret over all subsets (or $\sqrt{\log(ST)}$ with refined analysis using a Cauchy–Schwarz inequality). 

\section{Proof of \texorpdfstring{\cref{thm: regret non-uniform}}{}}\label{app: regret non-uniform}%
Define 
$$
\epsilon_{h,b}=\max\left\{1,\frac{1}{{r\db(b)}}\right\}\epsilon_h,\quad\text{for}~b\in[r].
$$
Also, for every phase $h$ and every level $n$, we define 
\begin{align*}
    \epst_{h,n}=\sqrt{\frac{\lg_h}{2^{n}}}.
\end{align*}
Next, we show the counterpart of \cref{lem: every arm is played} in the non-uniform case.

\begin{lemma}\label{lem: every arm is played for large m}
With probability at least $1-2\delta_h$, for any cluster $\cs\in\partition_h$, any block $b$ with %
$\cls(b)\subseteq \cs$, we have that: for any level $n\le\lceil \log(1/\epsilon_{h,b}^2)\rceil$, action $\ac\in\good_{h,n}(\cs)$, there exists $(i,j)\in\mathcal{D}_h$ satisfying $g(i)=b$. %
\end{lemma}
\begin{proof}
        Fix a block $b\in[r]$.
        For each accuracy level, recall in the step of data collection, we sample $\nes=r(\nc+\na)\lg_h2^n$ contexts i.i.d. from $\nu$.
        By~\cref{lem: sample iid rvs}, with probability $1-\frac{\delta_h}{r}$, at least $\nes\omega(\cl)/2$ of them are from block $b$ as
        $$
        \nes\omega(\cl)\ge \frac{8\nes
        }{S}\ge \lg \ge 16\log(r/\delta_h),
        $$
        where the first inequality comes from the near-uniform context distribution assumption.
    Since we add a context-action pair into $\mathcal{D}_n$ once we have collected
        $$2^n\le  2^{\lceil\log(1/\epsilon_{h,b}^2)\rceil}< 2^{\log(1/\epsilon_{h,b}^2)+1}=2/\epsilon_{h,b}^2$$
        samples for estimating its reward. Note that in the end, there are at most $|\cls(b)|(2^n-1)$ samples from block $b$ unused.
        Thus, with probability $1-\frac{\delta_h}{S}$, the number of the context-action pairs,  
        where the contexts are from block $b$, that are added into $\mathcal{D}_n$ is at least
        \begin{align*}
            \frac{\nes\omega(b)/2-|\cls(b)|(2^n-1)}{2^n} &\ge \frac{\nes\omega(b)/2}{2/\epsilon_{h,b}^2}-S\tag{$2^n\le 2/\epsilon_{h,b}^2$ and $|\cls(b)|\le S$}\\
            &\ge\frac{\nes\epsilon_h/(2r\epsilon_{h,b})}{2/\epsilon_{h,b}^2}-S \tag{by definition of $\epsilon_{h,b}$}\\
            &=16(\nc+\na)\log(rSK/\delta_h)-S\tag{the value of $\nes$}\\
            &\ge K\log(rSK/\delta_h).
        \end{align*}
        Conditioned on this event, with probability $1-\frac{\delta_h}{r}$, for any $\ac\in\good_{h,n}(\cs)$, there exists $(i,j)\in\mathcal{D}_n$ satisfying $g(i)=b$ by \cref{lem:sample actions}.
        Therefore, the lemma holds for block $b$ with probability at least $1-\frac{2\delta_h}{r}$
        We complete the proof by a union bound on all blocks. 
    \end{proof}

Now define a good event $\goodevent_h$ as
\begin{align*}
\goodevent_h &= \left\{ \left|\rxh_{h,n}(i,j) - \rx(i,j)\right|\le{\frac{1}{4}\epst_{h,n}}{},~\text{$\forall(i,j)\in \cD_h$,~$n\in[N_h]$} \right\} \;.
\end{align*}

\begin{lemma}\label{lem: good event holds - non-uniform}
Event $\goodevent_h$ holds with probability $1-\delta_h$.    
\end{lemma}
\begin{proof}
    Fix an $(i,j)$ pair and level $n$. Applying \cref{lem:subg}, we have with probability $1-\frac{\delta_h}{SKN_h}$,
    \begin{align*}
        \left|\rxh_{h,n}(i,j)-A(i,j)\right|\le 2\sqrt{\frac{\log(SKN_h/\delta_h)}{2^{n}}}\le\frac{1}{4}\epst_{h,n}.
    \end{align*}
    We complete the proof by applying a union bound on all $(i,j)$ pairs in $\cD_h$ and all levels $n\in[N_h]$.
    \end{proof}

In the following, we present and then prove the counterpart of \cref{lem: all jc are not eliminated} for \cref{alg: multiple blocks non-uniform}.
\begin{lemma}\label{lem: all jc are not eliminated - non-uniform}
    Assume action $\ac_\cl\in\good_{h,n}(\cs)$ for $\cl\in[\rk]$ with $\db(b)\ge\frac{2\epsilon_h}{r}$, and its corresponding cluster $\partition_{h} \ni\cs\supseteq\cls(\cl)$. Then $\ac_\cl$ is not eliminated from $\good_{h+1,n}(\cs)$ with probability at least $1-3\delta_h$ for any level $n\le\lceil \log(1/\epsilon_{h,b}^2)\rceil$.
\end{lemma}

\begin{proof}
    We prove by induction on $\lv$.
    The base case is $\lv=1$, which satisfies the condition of \cref{lem: every arm is played for large m} as
    $$
    \log\left(\frac{1}{\epsilon_{h,b}^2}\right)\ge\log\left(\frac{\db(b)^2r^2}{\epsilon_h^2}\right)\ge\log(4)\ge 1=n.
    $$
   For the inductive step we prove by contradiction.
   Assume that $\ac_\cl$ is eliminated in $\good_{h,n+1}(\cs)$ but not eliminated in $\good_{h,n}(\cs)$ for $\lv\ge 2$.
   This means that the following inequality hold:
   \begin{align*}
    \max_{\ac'\in\good_{h,n}(\cs)}\muu_{h,n}(\cs,\ac')-\muu_{h,n}(\cs,\ac_\cl)> 2\epst_{h,n}
   \end{align*}
   Let $\ac'=\argmax_{\ac\in\good(\lv,\cs)}\muu_{h,n}(\cs,\ac)$.
   By \cref{lem: every arm is played for large m} there exists a context $\cx'\in\cls(\cl)$ so that $\psi(\lv,\cx')\ni\ac'$ as ${\lv}\le \log(1/\epsilon_{h,b}^2)$.
   Moreover, the assumption that $\ac_\cl$ is eliminated implies that the condition at \cref{line: detect large gaps - non-uniform} does not hold for the current partition $\partition_h$, which further implies that
   $$
   \muu_{h,n}(\cs,\ac')-\rxh_{h,n}(\cx',\ac')<\epst_{h,n}
   $$
   On the other hand, by the assumption that $\ac_\cl$ is eliminated, again by \cref{lem: every arm is played for large m} we have that there exists a context $\cx''\in\cls(\cl)$, $\psi(\lv,\cx'')\ni\ac_\cl$ such that
   $$
    \muu_{h,n}(\cs,\ac')-\rxh_{h,n}(\cx'',\ac_\cl)>2 \epst_{h,n}
   $$
   Combining two inequalities, we have $\rxh_{h,n}(\cx'',\ac_\cl)+\epst_{h,n}<\rxh_{h,n}(\cx',\ac')$.
   This means that $\mu(\cl,\ac_\cl)+\frac{\epst_{h,n}}{2}<\mu(\cl,\ac')$ by \cref{lem: good event holds - non-uniform}, which contradicts the optimality of $\ac_\cl$.
   Therefore, we conclude that $\ac_\cl$ is not eliminated. 
\end{proof}

Now we present a key lemma similar to \cref{lem: all bads are eliminated - multiple blocks}.
\begin{lemma}\label{lem: all bads are eliminated - nonuniform}
For any block $\cl\in[\rk]$ with $\db(b)\ge\frac{2\epsilon_h}{r}$ and its corresponding cluster $\partition_h\ni\cs\supseteq\cls(\cl)$, all $3\epst_{h,n}$-suboptimal arms in block $\cl$ are eliminated in $\good_{h+1,n}(\cs)$ for any level $n\le\lceil \log(1/\epsilon_{h,b}^2)\rceil$ with probability at least $1-3\delta_h$. Consequently, only $6\epst_{h,n}$-optimal arms in block $\cl$ are played in phase $h+1$ for every context in block $b$. 
\end{lemma}

\begin{proof}%
    By \cref{lem: all jc are not eliminated - non-uniform}, $\ac_\cl$ is not eliminated with high probability, so for every pair of block $\cl\in[\rk]$ with $\db(b)\ge\frac{2\epsilon_h}{r}$ and arm $\ac\in\good_{h,n}(\cs)$ so that $\mu(\cl,\ac_\cl)-\mu(\cl,\ac)>3\epst_{h,n}$, we have
    \begin{align*}
        \max_{\ac'}\muu_{h,n}(\cs,\ac')-\muu_{h,n}(\cs,\ac)&\ge\muu_{h,n}(\cs,\ac_\cl)-\muu_{h,n}(\cs,\ac)\ge \mu(\cl,\ac_\cl)-\mu(\cl,\ac)-\frac{\epst_{h,n}}{2}\\
        &\ge 2.5\cdot \epst_{h,n}>2 \epst_{h,n}
    \end{align*}
    Therefore, $\ac$ is eliminated in $\good_{h,n}(\cs)$, and thus eliminated in $\good_{h+1,n'}(\cs)$ for $\lv'>\lv$. 
 \end{proof}

\subsection{Proof of \texorpdfstring{\cref{thm: regret non-uniform}}{}}
\begin{proof}
Fix a phase $h$ and a level $n$.
It suffices to bound regret within a single pair $(h,n)$.
For $\cl\in [\rk]$, let $n_b=\lceil\log(1/\epsilon_{h,n}^2)\rceil$. By \cref{lem: all bads are eliminated - nonuniform}, if $\lv>n_b$, all $6\epst_{h,n_b}$-suboptimal arms are eliminated from $\good_{h,n}(\cs)$ for any cluster $\cs\in\partition_h$. Therefore, regret of playing an action from $\good_{h,n}(\cs)$ is $6\epst_{h,n_b}$.
In this case, regret is bounded by
\begin{align*}
\sum_{\cl\in[\rk]}r(S+K)2^{(n+h)/2}\cdot \db(b)\cdot 6\epst_{h,n_b}&\le\sum_{\cl\in[\rk]}r(S+K)2^{(n+h)/2}\cdot \db(b)\cdot 6\sqrt{\frac{\lg_h}{2^{n_b}}}\\
&\le 6\sqrt{\lg_h}\sum_{\cl\in[\rk]}r(S+K)2^{(n+h)/2}\cdot \db(b)\cdot \epsilon_{h,b}\\
&\le 6\sqrt{\lg_h}\sum_{\cl\in[\rk]}r(S+K)2^{(n+h)/2}\cdot \db(b)\cdot \frac{\epsilon_{h}}{r\db(b)}\\
&\le 6\sqrt{\lg_h}r(S+K)2^{(n+h)/2}= \Otilde{r(S+K)2^{h/2}}
\end{align*}

If $\lv\le n_b$, a similar argument shows that regret of playing an action from $\good_{h,n}(\cs)$ is $6\epst_{h,n}$. Therefore, regret is bounded by
\begin{align*}
\sum_{\cl\in[\rk]}r(S+K)2^{(n+h)/2\cdot }\db(b)\cdot 6\epst_{h,n}&\le 6\sum_{\cl\in[\rk]}\db(b)\cdot r(S+K)2^{(n+h)/2}\sqrt{\frac{\lg_h}{2^n}}\\
&\le 6\sqrt{\lg_h}r(S+K)2^h= \Otilde{r(S+K)2^{h/2}}%
\end{align*}
We conclude the overall regret is $\Otilde{\sqrt{r(S+K)T}}$ for the data collection stage by noting that $T\ge r(S+K)2^{(h-1)}$ as phase $h-1$ is executed completely.
The same argument holds for analyzing the clustering stage when replacing $r$ with $r^3$, so we conclude the regret is bounded by $\Otilde{\sqrt{\rk^3 (\nc+\na)T}}$.
\end{proof}

\section{Omitted Details in \texorpdfstring{\cref{sec: low-rank}}{}}\label{app: low-rank}

In this section we discuss how to solve the more general low-rank bandit problem. Note that $\rx$ has rank $\rk$ if and only if there exist vectors $\wg_1,\dots,\wg_{\nc}\in\R^\rk$ and $\rv_1,\dots,\rv_{\na}\in\R^\rk$ so that $\rx(\cx,\ac)=\wg_{\cx}^\top\rv_{\ac}$.
We first consider $\rk=1$ and assume that every $\wg_\cx$ is non-negative.
In this case, we have
\begin{equation}
\argmax_{\ac\in[\na]}\rx(\cx,\ac)=\argmax_{\ac\in[\na]}\wg_\cx\rv_\ac=\argmax_{\ac\in[\na]}\rv_\ac.\label{eq: nonnegative argmax}
\end{equation}
In other words, there exists an arm that is optimal for all contexts.
The problem becomes simple as we can run the EXP4 algorithm \citep{auer2002nonstochastic} using \emph{constant} experts that recommend the same arm for all contexts. Thus, we will have only $\na$ experts and have the following proposition:
\begin{proposition}
Consider the rank-1 bandit problem with $\rk=1$ and $\wg_\cx\ge0$ for every $\cx\in[\nc]$. The regret of the EXP4 algorithm run with $\na$ constant experts is bounded by $O({\sqrt{\na T\log\na}})$.
\end{proposition}
However, the idea seems hard to generalize when $\wg_\cx\in \R$, as \cref{eq: nonnegative argmax} does not hold anymore and we need exponentially many experts for EXP4.
Next we introduce a new idea based on a reduction to context-lumpable bandits. %

In the following we define constant $B=\max_i\|\wg_i\|_\infty$.
To better illustrate the idea we first assume $r=1$ and consider the PAC setting.
We create an $\alpha$-covering of $[-B,B]$ and ``cluster'' each $\cx$ into one of the segment. 
Specifically, we have $\frac{2B}{\alpha}$ intervals $$\left[-B,-B+\frac{1}{\alpha}\right],~\left(-B+\frac{1}{\alpha},-B+\frac{2}{\alpha}\right],\dots,\left(B-\frac{1}{\alpha},B\right]$$ and each $\wg_\cx$ is assigned to the interval that contains it.
Given $\alpha$, let $\rr$ denote the number of intervals that have at least one context.
Conceptually we can view contexts in the same interval as if they are in the same block in context-lumpable bandits, and we have $\rr$ blocks analogously.
For contexts $\cx,\cx'$ in the same interval, they are indeed similar in the sense that we have $|\rx(\cx,\ac)-\rx(\cx,\ac)|=O(\alpha)$ for every arm $\ac$.
Intuitively, if $\alpha$ is much smaller than $\epsilon$, then \cref{alg:pac-unif-context} can proceed normally as a rank-$\rr$ context-lumpable bandit problem.
Consequently, we have the following theorem.
\begin{theorem}
    For $\rk=1$, by choosing $\alpha=\Theta(\epsilon)$, 
    \cref{alg:pac-unif-context} uses $\Otilde{\rr(S+K)/\epsilon^2}$ samples and outputs an $\Otilde{\epsilon}$-optimal policy.
\end{theorem}
Clearly we have $\rr\le\min\left\{\frac{1}{\alpha},\nc\right\}$, so the above theorem leads to a $\Otilde{(S+K)/\epsilon^3}$ sample complexity in the worst case.
The idea can be generalized to regret minimization and $\rk>1$.
Specifically, we construct an $\alpha$-grid of $[-B,B]^{\rk}$ so that $\rr=O(\tfrac{1}{\alpha^{\rk}})$ and run \cref{alg: multiple blocks non-uniform} for regret minimization.
Consequently, we have the following theorem:
\begin{theorem}
Let $p=\frac{1}{3\rk+2}$ and choose $\alpha=(\nc+\na)^pT^{-p}$. Then regret of \cref{alg: multiple blocks non-uniform} is bounded as $\rg_T=\Otilde{(\nc+\na)^{p}T^{1-p}}$.
\end{theorem}
\begin{proof}
Similar to previous analysis, the regret in a single phase $h$ is
\begin{align}\label{eq: low-rank regret}
\Otilde{\epsilon_h}\cdot \min\left\{T,\Otilde{\frac{\rr^3(\nc+\na)}{\epsilon_h^2}}\right\}=\min\left\{\Otilde{\epsilon_h T},\Otilde{\frac{\rr^3(\nc+\na)}{\epsilon_h}}\right\}
\end{align}
Recall that $\rr=O({1}/{\alpha^\rk})$. 
Therefore, when $\alpha=\Theta(\epsilon_h)$, we have the above regret is bounded by
$$
\Otilde{\frac{\rr^3(\nc+\na)}{\alpha}+\alpha T}=\Otilde{\frac{(\nc+\na)}{\alpha^{3r+1}}+\alpha T}=\Otilde{(\nc+\na)^{p}T^{1-p}}
$$
when choosing $\alpha$ optimally as $\alpha=(\nc+\na)^pT^{-p}$.
Otherwise, $\alpha=o(\epsilon_h)$ and \cref{eq: low-rank regret} can be bounded by
\begin{align*}
    \Otilde{\frac{\rr^3(\nc+\na)}{\epsilon_h}}=\Otilde{\frac{\rr^3(\nc+\na)}{\alpha}}=\Otilde{(\nc+\na)^{p}T^{1-p}}.
\end{align*}
We finish the proof by noting that there are at most $\log T$ phases.
\end{proof}
The bound becomes non-trivial when $\nc$ and $\na$ are large.
For example, when $\nc=\na=\sqrt{T}$, the bound is $T^{1-p/2}=o(T)$ for any $\rk$ while the $\Otilde{\sqrt{\nc\na T}}$ bound given by EXP4 is vacuous. 
The factor $3$ comes from the $\rk^3$ term in the regret bound of \cref{thm: regret non-uniform}.
It is a promising direction to first improve the factor in context-lumpable bandits and extend it to low-rank bandits using the reduction introduced here.

\end{document}